\documentclass{article}

\usepackage{microtype}
\usepackage{graphicx}
\usepackage{subcaption}
\usepackage{booktabs} % for professional tables
\usepackage[table]{xcolor}

\usepackage{wrapfig}

\usepackage{hyperref}
\usepackage{multirow}
\usepackage{algorithm}
\usepackage{algorithmic}

\usepackage[preprint]{format/neurips_2026}

\usepackage{amsmath}
\usepackage{amssymb}
\usepackage{mathtools}
\usepackage{amsthm}

\usepackage[capitalize,noabbrev]{cleveref}

\theoremstyle{plain}
\newtheorem{theorem}{Theorem}[section]
\newtheorem{proposition}[theorem]{Proposition}
\newtheorem{lemma}[theorem]{Lemma}

\theoremstyle{definition}

\theoremstyle{remark}

\usepackage[disable,textsize=tiny]{todonotes}

\usepackage{placeins} % fix figure place

\title{GAC: Stabilizing Asynchronous RL Training for LLMs via Gradient Alignment Control}

\author{%
Haofeng Xu$^{1,2}$ \quad
Junwei Su$^{1,\dagger}$ \quad
Yukun Tian$^{3,2}$ \quad
Lansong Diao$^{2}$ \quad
Zhengping Qian$^{2}$ \quad
Chuan Wu$^{1,\dagger}$ \\
\small
$^{1}$School of Computing and Data Science, The University of Hong Kong \\
\small
$^{2}$Alibaba Group \quad
$^{3}$Southeast University \\
\small
$^{\dagger}$Corresponding authors. \\
\small
\texttt{junweisu.cs@gmail.com}, \quad \texttt{cwu@cs.hku.hk}
}

\begin{document}

\maketitle

\begin{abstract}
Asynchronous execution is essential for scaling reinforcement learning (RL) to modern large-model workloads, including large language models and AI agents, but it can fundamentally alter RL optimization behavior. While prior work on asynchronous RL focuses on training throughput and distributional correction, we show that naïvely applying asynchrony to policy-gradient updates can induce qualitatively different training dynamics and lead to severe training instability. Through systematic empirical and theoretical analysis, we identify a key signature of this instability: asynchronous training exhibits persistently high cosine similarity between consecutive policy gradients, in contrast to the near-orthogonal updates observed under synchronized training. This \emph{stale-aligned gradient} effect amplifies correlated updates and increases the risk of overshooting and divergence. Motivated by this observation, we propose \textsc{Gradient Alignment Control} (\textsc{GAC}), a simple dynamics-aware stabilization method that regulates asynchronous RL progress along stale-aligned directions via gradient projection. We establish convergence guarantees under bounded staleness and demonstrate empirically that \textsc{GAC} recovers stable, on-policy training dynamics and matches synchronized baselines even at high staleness.
\end{abstract}
\section{Introduction}
\vspace{-4pt}
Reinforcement learning (RL) provides a principled framework for optimizing large language models (LLMs) as decision-makers, where text generation is modeled as a sequence of actions and learning is driven by feedback on complete responses rather than token-level supervision~\citep{zhang2025surveyreinforcementlearninglarge,wang2025reinforcementlearningenhancedllms}. 
In modern alignment pipelines, policy-gradient methods such as Group Relative Policy Optimization (GRPO)~\citep{shao2024deepseekmath} train an LLM by sampling multiple candidate responses to the same prompt and reinforcing those that achieve higher relative reward within the group, directly optimizing response-level quality and reasoning performance. This RL-based fine-tuning paradigm has become central to aligning LLMs with human intent, improving mathematical and logical reasoning, and adapting pretrained models to complex, open-ended tasks~\citep{feng2025group}.

Despite this algorithmic flexibility, the scalability of RL for LLMs is increasingly constrained by \emph{system-level bottlenecks}~\citep{sheng2024hybridflow,zhong2025streamrlscalableheterogeneouselastic}. A typical RL training pipeline consists of several heterogeneous stages—rollout generation, reward computation, and policy optimization—that operate on different time scales and rely on distinct hardware resources. As model sizes grow, enforcing strict synchronization leads to significant resource underutilization~\citep{fu2025areal,wang2025reinforcement}: fully synchronous methods such as GRPO require all workers to complete rollout collection before any update can proceed, exacerbating straggler effects and resulting in poor wall-clock efficiency~\citep{sheng2025laminarscalableasynchronousrl}. To alleviate these issues, modern RL systems increasingly adopt asynchronous training, in which rollout generation and policy optimization proceed concurrently. While essential for scaling, this relaxation fundamentally alters the learning dynamics of policy optimization—introducing new sources of bias and instability that must be carefully understood and controlled.

\paragraph{Limitations and Gaps in Existing Work.}
There has been growing interest in asynchronous RL as a means of addressing scalability challenges, but existing efforts leave a critical gap in understanding the \emph{learning dynamics} induced by asynchrony. Prior work falls into two categories. \emph{Systems-oriented} approaches~\citep{fu2025areal,sheng2025laminarscalableasynchronousrl,lu2025iirollflash} focus on engineering benefits such as mitigating stragglers and maximizing throughput, typically treating the learning algorithm as a black box. \emph{Algorithm-oriented} approaches~\citep{xi2025bapo,zheng2025m2po,li2026a3poacceleratingasynchronousllm} emphasize \emph{distributional corrections} to reduce the mismatch between on-policy and stale-policy distributions. While such corrections can be effective when training dynamics remain stable, they overlook a more fundamental issue: asynchronous and synchronous training can exhibit \emph{qualitatively different training dynamics}. As we demonstrate, naïvely applying asynchronous execution to GRPO-style updates can lead to severe instability and training collapse.

\paragraph{Contributions.} Motivated by these observations, we study asynchronous RL from an underexplored \emph{training-dynamics perspective} and identify a simple yet effective control mechanism that directly targets these dynamics. Our main contributions are:
\begin{enumerate}
    \item \textbf{Characterization of asynchronous training dynamics.}
We conduct a systematic empirical comparison between synchronous and asynchronous RL and uncover two consistent findings: asynchronous training is substantially more prone to \emph{collapse}, and this instability is associated with persistently \emph{high cosine similarity} between consecutive policy gradients—in contrast to the low similarity observed under synchronous training. This alignment explains instability: gradient staleness repeatedly reinforces similar update directions, increasing the risk of overshooting and divergence. We further support this with a theoretical analysis (Theorem~\ref{thm:align_async_0}) showing that stale gradients introduce systematic bias that increases gradient alignment and fundamentally alters convergence behavior.

\item \textbf{A dynamics-aware stabilization method with convergence guarantees.}
We propose \textsc{GAC} (Gradient Alignment Control), a mechanism that stabilizes asynchronous RL by explicitly regulating progress along stale-aligned update directions via adaptive gradient projection. We provide theoretical analysis (Proposition~\ref{prop:bias_reduction_0}) showing that \textsc{GAC} reduces alignment-induced bias and improves convergence stability. \textsc{GAC} is algorithm- and system-agnostic, and can be readily integrated with existing RL algorithms and distributed training systems.

\item \textbf{End-to-end improvements in asynchronous GRPO training.} We evaluate \textsc{GAC} across Qwen3 models (1.7B--8B) and seven mathematical reasoning benchmarks. Under moderate to large staleness, na\"{\i}ve asynchronous GRPO and baselines exhibit unstable dynamics and frequent collapse. In contrast, \textsc{GAC} consistently stabilizes training, restores on-policy-like optimization behavior, and substantially closes the performance gap to synchronized GRPO, while preserving throughput benefits.
\end{enumerate}
\vspace{-4pt}

\section{Preliminaries}
\vspace{-4pt}
\subsection{Group-Relative Policy Optimization (GRPO)}
GRPO is a policy-gradient method designed for reinforcement learning with outcome-level or verifier-based rewards, and is widely used in large language model post-training. GRPO follows the clipped-surrogate optimization framework of Proximal Policy Optimization (PPO)~\citep{schulman2017ppo}, but replaces value-function--based advantages with a group-relative normalization computed over multiple responses to the same prompt.

Let $\pi_\theta(a \mid s)$ denote a stochastic policy parameterized by $\theta$, and $\pi_{\theta_b}$ denote the behavior policy used to generate training data. For a sampled state--action pair $(s,a)$, define the importance ratio
\[
r_\theta(s,a) \;\triangleq\; \frac{\pi_\theta(a \mid s)}{\pi_{\theta_b}(a \mid s)}.
\]
Given an advantage estimate $A$, GRPO adopts the PPO clipped surrogate
\[
\ell_\varepsilon(r,A)
\;\triangleq\;
\min\!\big(rA,\; \mathrm{clip}(r,1-\varepsilon,1+\varepsilon)\,A\big).
\]
For a prompt or context $c \sim \mathcal{D}$, GRPO samples a group of $G$ trajectories $\{\tau^{(g)}\}_{g=1}^G$, producing scalar returns $\{R^{(g)}\}_{g=1}^G$. The group-relative advantage is defined as $A_b^{(g)} \triangleq (R^{(g)} - \mu_G)/\sigma_G$, where $\mu_G$ and $\sigma_G$ are the mean and standard deviation of returns within the group. This normalization serves as an implicit control variate, reducing gradient variance without requiring a learned value function.

Let $d_{\pi_{\theta_b}}(s,a \mid c)$ denote the discounted state--action visitation distribution induced by the behavior policy under context $c$. The GRPO objective is given by
\[
\mathcal{L}_{\mathrm{GRPO}}(\theta)=
\mathbb{E}_{c \sim \mathcal{D}}
\left[
\frac{1}{G}
\sum_{g=1}^G
\mathbb{E}_{(s,a)\sim d_{\pi_{\theta_b}}(\cdot \mid c)}
\Big[
\min\!\Big(
r_\theta(s,a)\, A_b^{(g)},\;
\mathrm{clip}\big(r_\theta(s,a),\, 1\!-\!\varepsilon,\, 1\!+\!\varepsilon\big)\, A_b^{(g)}
\Big)
\Big]
\right].
\]
The clipping operation enforces a trust-region--like constraint on policy updates, while group-relative normalization provides context-aware credit assignment suitable for sparse or outcome-based rewards.

\subsection{Synchronous and Asynchronous GRPO Training}

We distinguish \emph{synchronous} from \emph{asynchronous} GRPO by whether rollout generation and policy optimization are coupled, consistent with standard distributed RL pipelines~\citep{mnih2016a3c,espeholt2018impala}.

In \emph{synchronous GRPO}, rollout generation and optimization proceed in lockstep. At update step $t$, trajectories are generated using the current policy $\pi_{\theta_t}$, and the aggregated policy gradient
\[
g_t \;\triangleq\; \nabla_\theta \mathcal{L}_{\mathrm{GRPO}}(\theta_t)
\]
is computed from on-policy data. The policy parameters are updated as
\[
\theta_{t+1} = \theta_t + \eta g_t,
\]
where $\eta > 0$ is the learning rate. In this regime, gradients are evaluated under the current policy distribution, and standard on-policy optimization behavior is observed.

In \emph{asynchronous GRPO}, rollout generation is decoupled from optimization to improve system throughput. At optimization step $t$, the data batch may be generated by a stale behavior policy $\pi_{\theta_{t-\tau}}$, where $\tau \ge 0$ denotes the staleness induced by pipeline delays and straggler effects. The learner update takes the form
\[
\theta_{t+1} = \theta_t + \eta\, \hat g_t,
\]
where $\hat g_t$ is a stochastic gradient estimator constructed from trajectories sampled under $d_{\pi_{\theta_{t-\tau}}}$.

Asynchronous execution decouples rollout generation from policy optimization to maximize throughput, but fundamentally transforms GRPO into an off-policy procedure. Although importance ratios partially correct for the behavior policy at the surrogate level, gradients are nonetheless evaluated under a lagged distribution, inducing systematic bias and temporal coupling between successive updates that does not arise in synchronized training. Prior work attributes such failures primarily to distributional mismatch manifesting as trust-region violations; we show that the resulting dynamical effects run deeper and can qualitatively alter optimization behavior. Understanding and controlling these dynamics is the central focus of this work.
\vspace{-4pt}
\section{Training Dynamics of Asynchronous RL}
\label{sec:motivation}
\vspace{-4pt}
In this section, we present one of the central findings of this paper: asynchronous and synchronous RL exhibit \emph{qualitatively different training dynamics}. We begin with an empirical study demonstrating that staleness introduced by asynchronous execution can induce catastrophic instability in policy optimization. We then show that this instability is consistently preceded by a clear dynamical signature—persistent alignment between consecutive policy gradients—captured by their cosine similarity. In contrast, synchronized RL typically exhibits low inter-gradient similarity, with near-orthogonal update directions. Finally, we provide a theoretical analysis that grounds these empirical observations, showing how gradient staleness induces alignment and fundamentally alters the convergence behavior of asynchronous RL.

\begin{figure*}[t]
  \centering
  \begin{subfigure}[t]{0.24\textwidth}
    \centering
    \includegraphics[width=\linewidth]{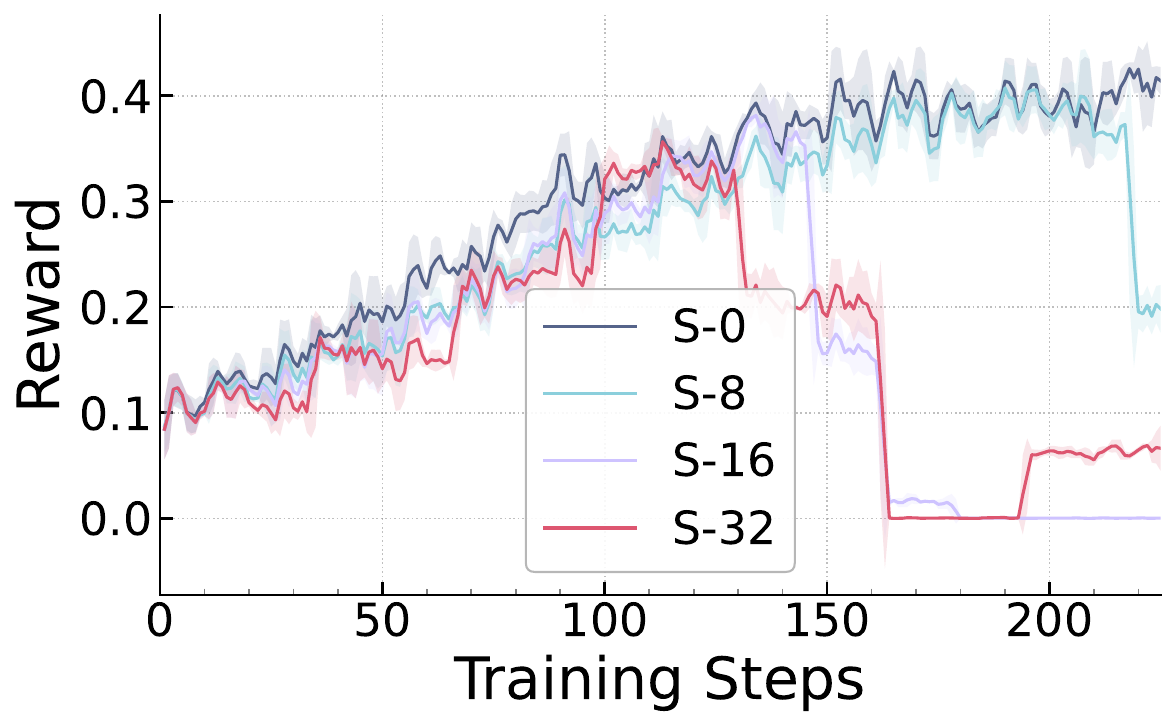}
    \caption{Reward trajectory.}
    \label{fig:motivation-staleness-reward}
  \end{subfigure}\hfill
  \begin{subfigure}[t]{0.24\textwidth}
    \centering
    \includegraphics[width=\linewidth]{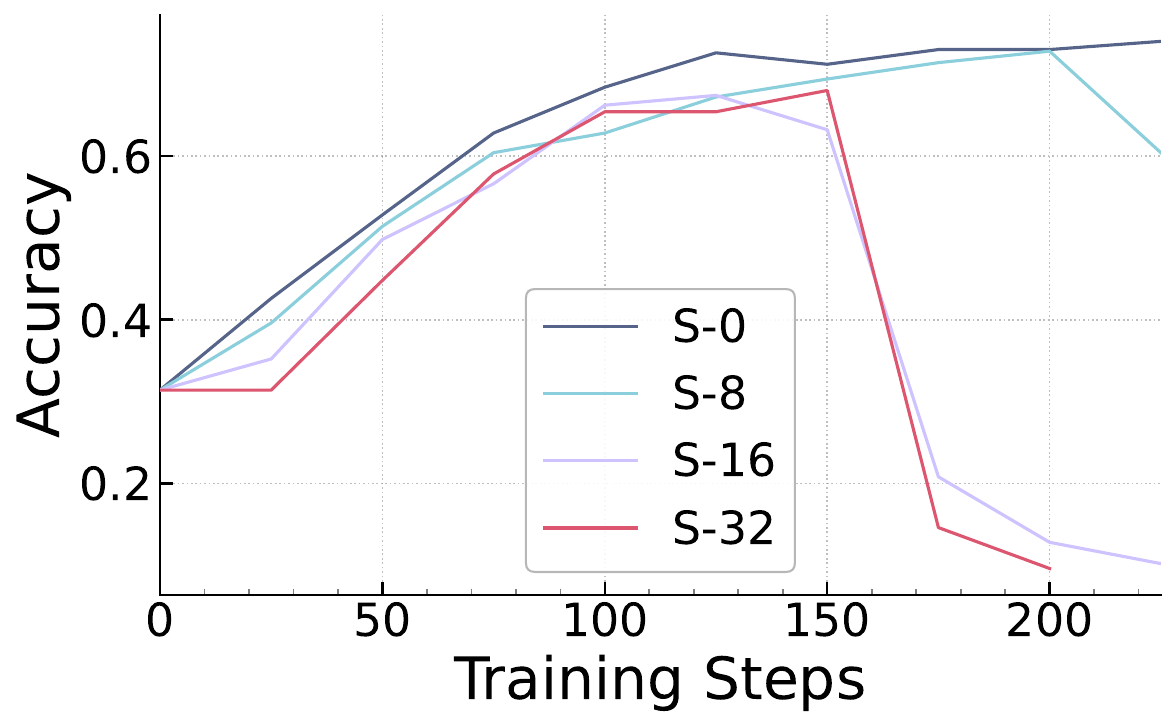}
    \caption{Accuracy trajectory.}
    \label{fig:motivation-staleness-acc}
  \end{subfigure}\hfill
  \begin{subfigure}[t]{0.24\textwidth}
    \centering
    \includegraphics[width=\linewidth]{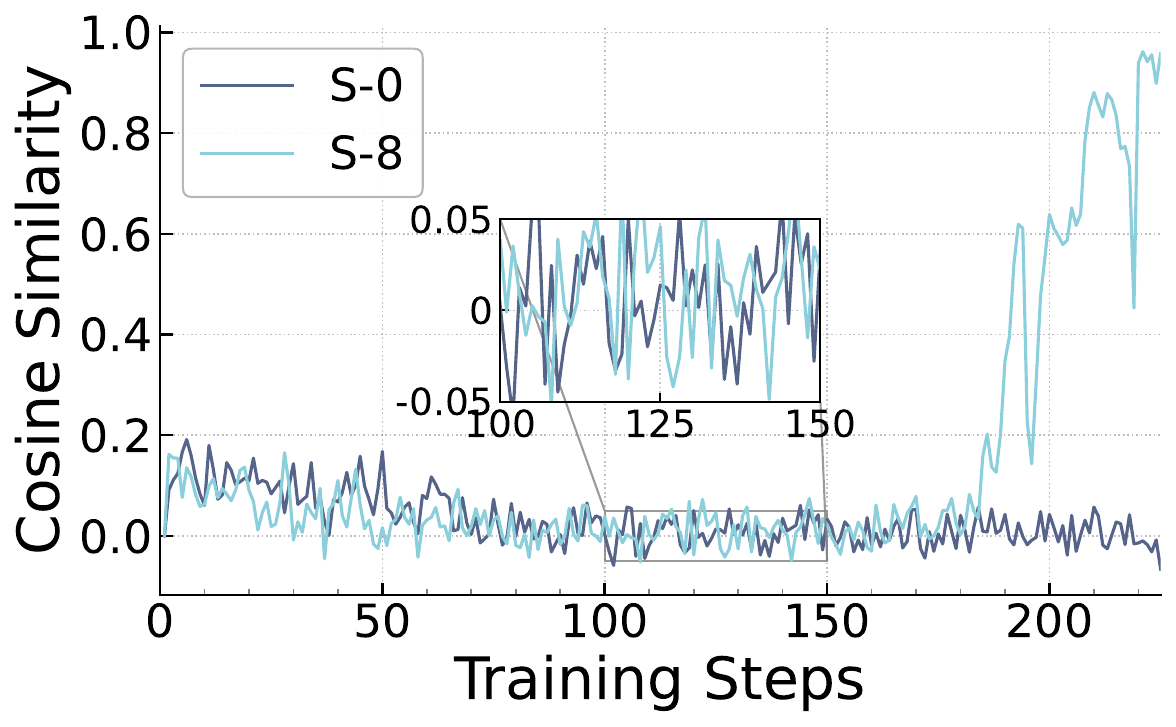}
    \caption{$c_t$ ($s{=}0$,\; $s{=}8$).}
    \label{fig:motivation-staleness-cossim-s0-s8}
  \end{subfigure}\hfill
  \begin{subfigure}[t]{0.24\textwidth}
    \centering
    \includegraphics[width=\linewidth]{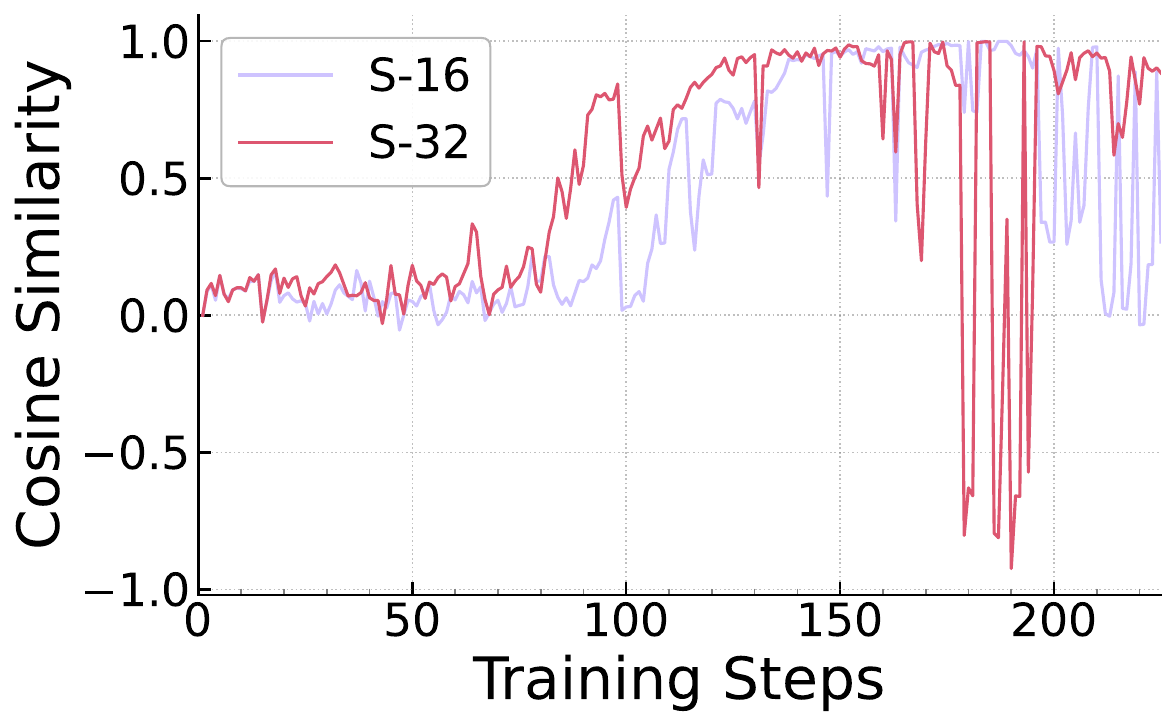}
    \caption{$c_t$ ($s{=}16$,\; $s{=}32$).}
    \label{fig:motivation-staleness-cossim-s16-s32}
  \end{subfigure}

  \caption{%
  \textbf{Progressive instability under increasing staleness.}
  Panels (a--b) report training reward and validation accuracy across staleness levels.
  Panels (c--d) report consecutive-gradient cosine similarity $c_t$, contrasting synchronized training ($s{=}0$) with mild ($s{=}8$) and larger staleness ($s{=}16,32$).
  Results on MATH with Qwen3-1.7B.
  }
  \label{fig:motivation-staleness}
\end{figure*}

% =========================================
% SUBSECTION: Problem Severity
% =========================================
\subsection{Staleness Induces Catastrophic Instability}
\label{subsec:problem-severity}
We study the effect of staleness by varying $s \in \{0, 8, 16, 32\}$, where $s$ denotes the number of optimization steps by which the behavior policy lags behind the learner policy, keeping all other training configurations fixed.
We define \emph{training collapse} as persistent instability characterized by abrupt large-magnitude drops in both reward and accuracy, no subsequent recovery over extended training, and substantial performance degradation relative to the on-policy reference.
Figure~\ref{fig:motivation-staleness} (a)--(b) illustrate this progressive degradation. Synchronized training ($s{=}0$) converges stably; mild staleness ($s{=}8$) remains stable until around step 220 before collapsing; at larger staleness, collapse occurs earlier, with $s{=}16$ collapsing around step 150 and $s{=}32$ around step 130, with no subsequent recovery in either case. This progressive degradation reveals a fundamental stability limitation of asynchronous GRPO, motivating the need for explicit stabilization.

\subsection{Gradient Alignment as a Dynamical Signature}
\label{subsec:dynamical-signature}

Having established that asynchronous GRPO can fail catastrophically, we now identify \emph{what distinguishes stable and unstable training dynamics}.
To this end, we introduce the \emph{consecutive-gradient cosine similarity}
{\small
\begin{equation}
\label{eq:cos-sim}
c_t \triangleq \frac{\langle \hat{g}_t, \hat{g}_{t-1}\rangle}{\|\hat{g}_t\|_2\,\|\hat{g}_{t-1}\|_2},
\end{equation}
}
where $\hat{g}_t$ denotes the aggregated policy gradient estimate used at update step $t$, computed from the update batch and potentially induced by staled rollout policy (see Appendix~\ref{app:ct-computation}).
This metric captures whether successive updates explore diverse descent directions ($|c_t|\approx 0$) or  exhibit strong directional alignment across iterations ($c_t\approx 1$).

\paragraph{Distinct cosine-similarity behavior under synchronous and asynchronous training.}
We empirically find that consecutive-gradient cosine similarity exhibits a clear qualitative difference between synchronous and asynchronous RL, as illustrated in Figure~\ref{fig:motivation-staleness} (c)--(d). 
In synchronous RL, cosine similarity remains tightly concentrated near zero throughout training, indicating that consecutive policy-gradient updates are nearly orthogonal and explore diverse descent directions—behavior that is stable across training stages and model scales (additional on-policy results across model sizes are provided in Appendix~\ref{app:onpolicy-ct}).

In contrast, asynchronous RL exhibits persistently large-magnitude cosine similarity: under staleness, $|c_t|$ departs from zero and becomes markedly more volatile. At mild staleness ($s{=}8$), $|c_t|$ spikes around step 180, well before reward degradation at step 220; at $s{=}16$, alignment rises as early as step 80, preceding collapse at step 150 by a substantial margin. At larger staleness, $|c_t|$ oscillates with very large amplitude, indicating repeated reinforcement of similar update directions. This sustained alignment provides a mechanistic explanation for collapse: correlated updates amplify effective step sizes along a low-dimensional subspace, causing errors to accumulate and overshoot the trust region. Crucially, the consistent precedence of $|c_t|$ elevation over reward collapse identifies gradient alignment as a leading indicator---and likely a causal driver---of instability, establishing it as a robust stability signature that motivates explicit alignment control in asynchronous RL.

\subsection{Theoretical Evidence}
\label{subsec:theory-motivation}
We now show that the empirical gradient-alignment phenomenon admits a direct theoretical explanation.
The key observation is that asynchronous execution fundamentally changes the structure of the stochastic gradient update.
In particular, stale rollouts can introduce a \emph{persistent bias} that couples consecutive gradients, causing the optimization trajectory to deviate from standard descent behavior.

\begin{theorem}[Convergence of Asynchronous GRPO, Informal]
\label{thm:align_async_0}
Under standard smoothness and bounded-variance assumptions, and assuming locally bounded parameter drift induced by staleness, the iterates satisfy
{\small
\begin{align*}
&\min_{0\le t\le T-1}\mathbb{E}\|\nabla \mathcal{L}(\theta_t)\|^2
\;\le\;
\frac{2(L^\star-\mathcal{L}(\theta_0))}{\eta T}
+2L\eta\sigma^2 +\mathrm{Err}(T)  \\
&\quad +\frac{4L\eta}{T}\sum_{t=0}^{T-1}\mathbb{E}\|b_t\|^2
-\frac{2}{T}\sum_{t=0}^{T-1}\eta\rho_t\lambda_t\,
\mathbb{E}\|\widehat g_{t-1}\|^2,
\end{align*}
}
\noindent where $L^\star \triangleq \sup_\theta \mathcal{L}(\theta)$,
$L$ is the smoothness constant of $\mathcal{L}(\theta)$,
$\sigma^2$ bounds the variance of the gradient,
$b_t$ denotes the staleness-induced bias arising from stale rollouts,
and $\widehat g_{t-1}$ is the stochastic gradient from the previous update.
The coefficients $\rho_t$ and $\lambda_t$ quantify the persistence and strength of temporal coupling induced by staleness, respectively.
Under the small-drift linearization, the bias admits the approximation
$
b_t \;\approx\; \eta_{t-1} B_t \widehat g_{t-1},
$
where $B_t$ is a local linear operator mapping parameter staleness into gradient bias.
The term $\mathrm{Err}(T)$ collects higher-order remainder terms due to rollout noise and approximation error.
\end{theorem}

The formal version of Theorem~\ref{thm:align_async_0} and its proof can be found in Appendix~\ref{app:async-convergence}.
Theorem~\ref{thm:align_async_0} extends the nonconvex convergence guarantee with two staleness-dependent effects.
The term $\sum_t \mathbb{E}\|b_t\|^2$ captures the \emph{magnitude} of the staleness-induced bias and is always detrimental.
The final term reflects \emph{temporal coupling}: when the bias aligns with previous gradients, consecutive updates become correlated.

\paragraph{Alignment-induced instability and collapse.}
Under asynchrony, empirical results show that this temporal coupling strengthens over time, leading to persistent alignment between consecutive gradients.
Through the approximation $b_t \approx \eta_{t-1} B_t \widehat g_{t-1}$, such alignment directly inflates the bias magnitude $\|b_t\|$ as past gradients are repeatedly reinforced.
As a result, the bias penalty term $\sum_t \mathbb{E}\|b_t\|^2$ grows and eventually dominates the right-hand side of the bound, overwhelming the descent signal.
When this occurs, updates no longer behave as local descent steps, leading to unstable dynamics and training collapse.
This transition manifests empirically as a sharp rise in consecutive-gradient cosine similarity immediately preceding performance degradation. The analysis identifies \emph{alignment persistence}—rather than distributional mismatch alone—as the primary driver of instability in asynchronous GRPO. This perspective motivates explicit control of inter-step gradient alignment to suppress bias amplification and prevent collapse, which directly informs the design of our method below.

\section{Gradient Alignment Control (GAC)}
\label{sec:method}
\vspace{-4pt}
We propose \textsc{GAC}, a dynamics-aware stabilization mechanism for asynchronous GRPO that explicitly regulates \emph{inter-step gradient alignment} induced by stale rollouts.
Unlike prior approaches that operate at the surrogate or distributional level, \textsc{GAC} intervenes directly at the level of the policy-gradient update.
The method monitors the alignment between consecutive gradients and selectively suppresses progress along stale-aligned directions while preserving informative orthogonal components.
This intervention breaks the positive feedback loop created by temporal coupling under staleness and restores update geometry closer to that observed in synchronized training. The complete procedure of GAC is presented at Algorithm~\ref{alg:gac}.

\subsection{Directional Gradient Projection}
\label{sec:method:projection}

The core operation of \textsc{GAC} is a directional projection that rescales gradient components relative to a reference direction.
Let $v$ be a nonzero reference vector and $u \triangleq v/\|v\|_2$ its unit direction.
Given any gradient $g$, we decompose it into components parallel and orthogonal to $u$:
\begin{equation}
g_{\parallel} \;\triangleq\; \langle g,u\rangle u,
\qquad
g_{\perp} \;\triangleq\; g - g_{\parallel}.
\label{eq:decomp}
\end{equation}
\textsc{GAC} applies anisotropic rescaling with directional gains $(\alpha,\beta)\in\mathbb{R}_+^2$:
\begin{equation}
g' \;\triangleq\; \alpha\, g_{\parallel} + \beta\, g_{\perp}.
\label{eq:projection}
\end{equation}
Choosing $\alpha < \beta$ attenuates progress along the reference direction $u$ while preserving the residual update signal in the orthogonal subspace.
This operation can be equivalently written as a rank-one linear transform
$
g' = \bigl(\beta I + (\alpha-\beta)uu^\top\bigr) g,
$
highlighting that \textsc{GAC} modifies update geometry rather than gradient magnitude uniformly.

\subsection{Adaptive Alignment Control under Stale Rollouts}
\label{sec:method:algorithm}

In asynchronous GRPO, gradients $\widehat g_t$ are computed from rollouts generated by a stale behavior policy $\pi_{\theta_{t-\tau}}$, which induces temporal coupling across updates.
To regulate this effect, \textsc{GAC} tracks the cosine similarity $c_t$ (Eq.~\ref{eq:cos-sim}) between consecutive gradients, which serves as a direct measure of alignment.
As shown in Section~\ref{sec:motivation}, synchronized training maintains $|c_t|$ near zero, whereas asynchronous training exhibits persistently elevated alignment prior to collapse. \textsc{GAC} partitions the alignment space using two thresholds $0 < c_{\text{low}} < c_{\text{high}}$, corresponding to three operational regimes:

\vspace{0.4em}
\noindent\textbf{Safe regime} ($|c_t| \le c_{\text{low}}$).  
Alignment remains within the on-policy range.
No intervention is applied and the update proceeds normally:
$
\theta_{t+1} = \theta_t + \eta\,\widehat g_t.
$

\vspace{0.25em}
\noindent\textbf{Projection regime} ($c_{\text{low}} < |c_t| < c_{\text{high}}$).  
Alignment is elevated but correctable.
\textsc{GAC} applies directional projection to reduce alignment to the target level $c_{\text{low}}$.

\vspace{0.25em}
\noindent\textbf{Violation regime} ($|c_t| \ge c_{\text{high}}$).  
Alignment indicates a severe trust-region violation.
The update is skipped to prevent destabilizing overshoot:
$
\theta_{t+1} = \theta_t.
$

\vspace{0.4em}
\noindent
In the projection regime, \textsc{GAC} uses the previous gradient as the reference direction,
$u_{t-1} \triangleq \widehat g_{t-1}/\|\widehat g_{t-1}\|_2$,
and rescales the parallel component to achieve the target alignment:
\begin{equation}
\widehat g_t'
\;=\;
\frac{c_{\text{low}}}{|c_t|}\,
\langle \widehat g_t, u_{t-1}\rangle u_{t-1}
\;+\;
\Bigl(
\widehat g_t
-
\langle \widehat g_t, u_{t-1}\rangle u_{t-1}
\Bigr).
\label{eq:gac_adaptive}
\end{equation}
This rescaling preserves orthogonal gradient information while reducing the magnitude of the alignment component.

\subsection{Provable Benefit}
\label{sec:method:theory}

We now connect the projection mechanism to the theoretical analysis of staleness-induced bias, $b_t$.
The following proposition formalizes how \textsc{GAC} attenuates this effect.

\begin{proposition}[Bias reduction via orthogonal projection, Informal]
\label{prop:bias_reduction_0}
Assume the \emph{small-drift linearization} and \emph{temporal persistence} conditions hold, such that
$
b_t \approx \eta_{t-1} B_t \widehat g_{t-1}.
$
Let $b_t^\perp$ denote the effective bias after applying \textsc{GAC} with $\alpha \le \beta$.
Then
\begin{equation}
\mathbb{E}\|b_t^\perp\|^2
\;\le\;
\mathbb{E}\|b_t\|^2
-
\eta^2 \lambda_t^2\,\mathbb{E}\|\widehat g_{t-1}\|^2
+
\mathcal{O}(\epsilon_t),
\end{equation}
where $\lambda_t>0$ quantifies the alignment strength between $b_t$ and $\widehat g_{t-1}$, and $\epsilon_t$ collects higher-order approximation terms.
\end{proposition}

\paragraph{Interpretation.} The formal version of Proposition~\ref{prop:bias_reduction_0} and its proof can be found in Appendix~\ref{app:prop}. 
Proposition~\ref{prop:bias_reduction_0} shows that suppressing the previous-step direction yields a \emph{strict reduction} in the magnitude of the staleness-induced bias.
The removed term scales with the same persistence factor $\lambda_t$ that drives gradient alignment under stale rollouts.
Consequently, \textsc{GAC} directly weakens the mechanism responsible for pre-collapse amplification, providing a principled explanation for the stabilization effects observed empirically.
\vspace{-4pt}
\section{Experiments}
\label{sec:experiments}
\vspace{-4pt}
In this section, we conduct a comprehensive evaluation of \textsc{GAC} for stabilizing off-policy GRPO training induced by stale rollouts.
Across experiments, \textsc{GAC} remains stable at large staleness, consistently outperforms off-policy baselines, and recovers on-policy-level performance.
Implementation details, additional results and hyperparameter choices are provided in Appendix~\ref{app:impl}.

\begin{figure*}[t]
  \centering
  % =========================
  % Row 1: Reward (4 models)
  % =========================
  \begin{subfigure}[t]{0.24\textwidth}
    \centering
    \includegraphics[width=\linewidth]{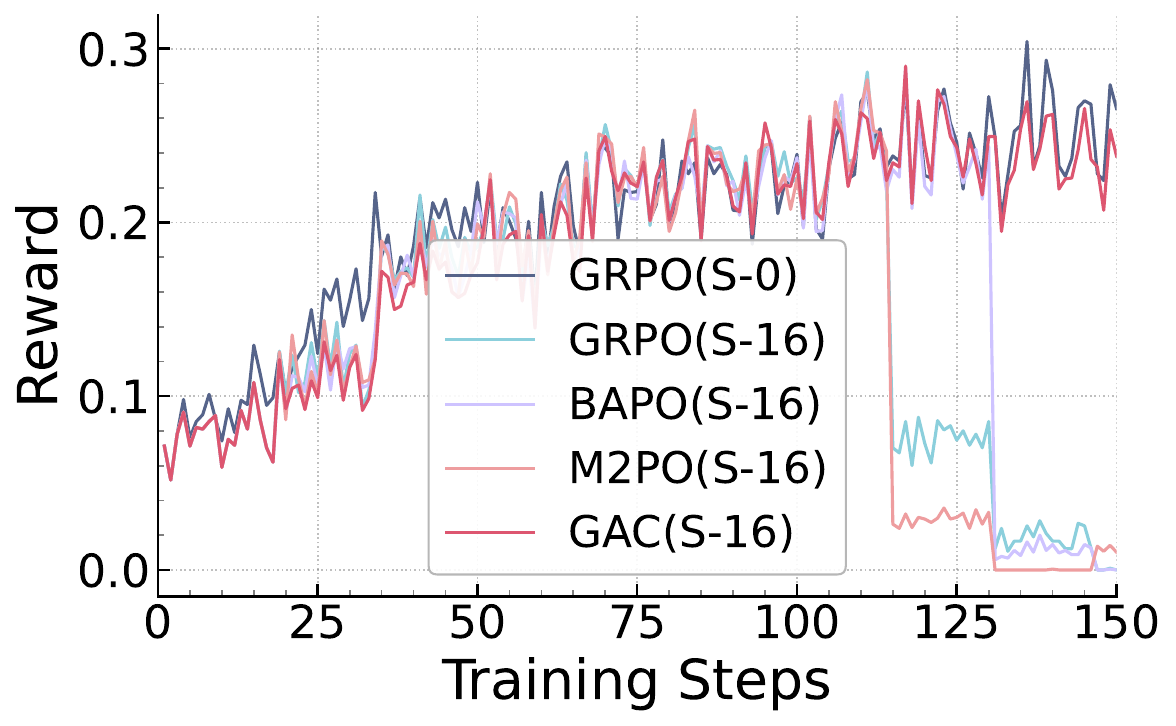}
    \caption{\small Llama-3.2-3B-Inst.}
    \label{fig:dynamics-llama-3p2-3b-reward}
  \end{subfigure}
  \begin{subfigure}[t]{0.24\textwidth}
    \centering
    \includegraphics[width=\linewidth]{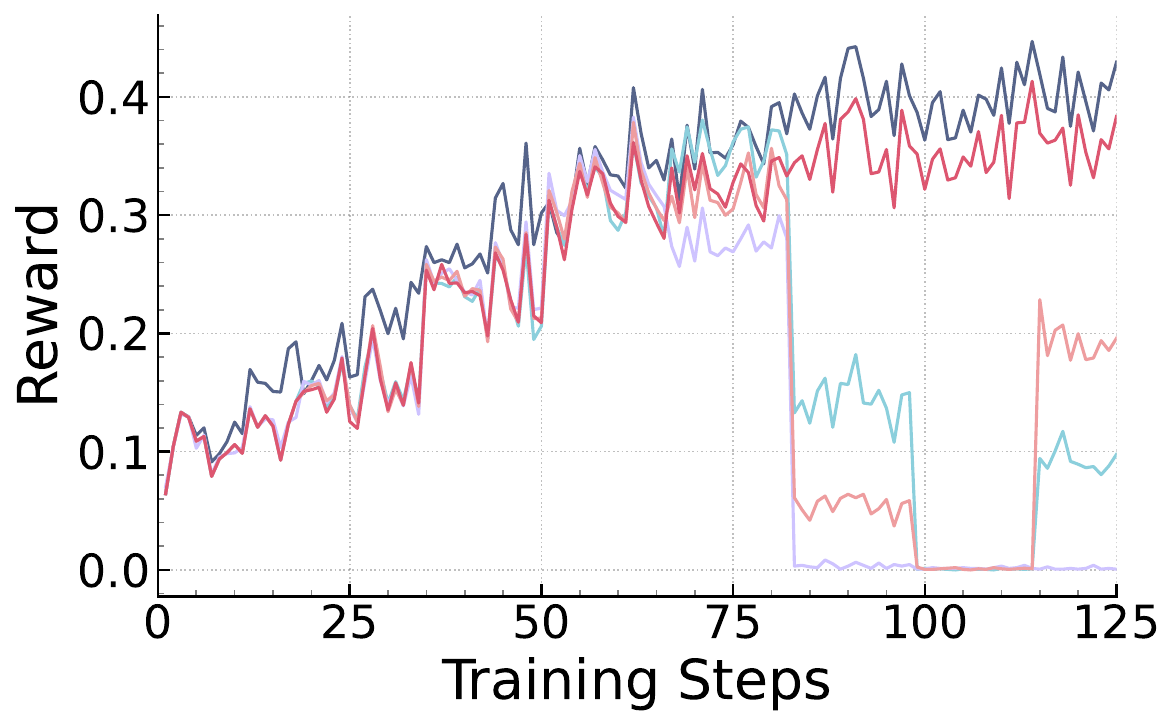}
    \caption{\small Qwen3-1.7B}
    \label{fig:dynamics-qwen-1p7b-reward}
  \end{subfigure}\hfill
  \begin{subfigure}[t]{0.24\textwidth}
    \centering
    \includegraphics[width=\linewidth]{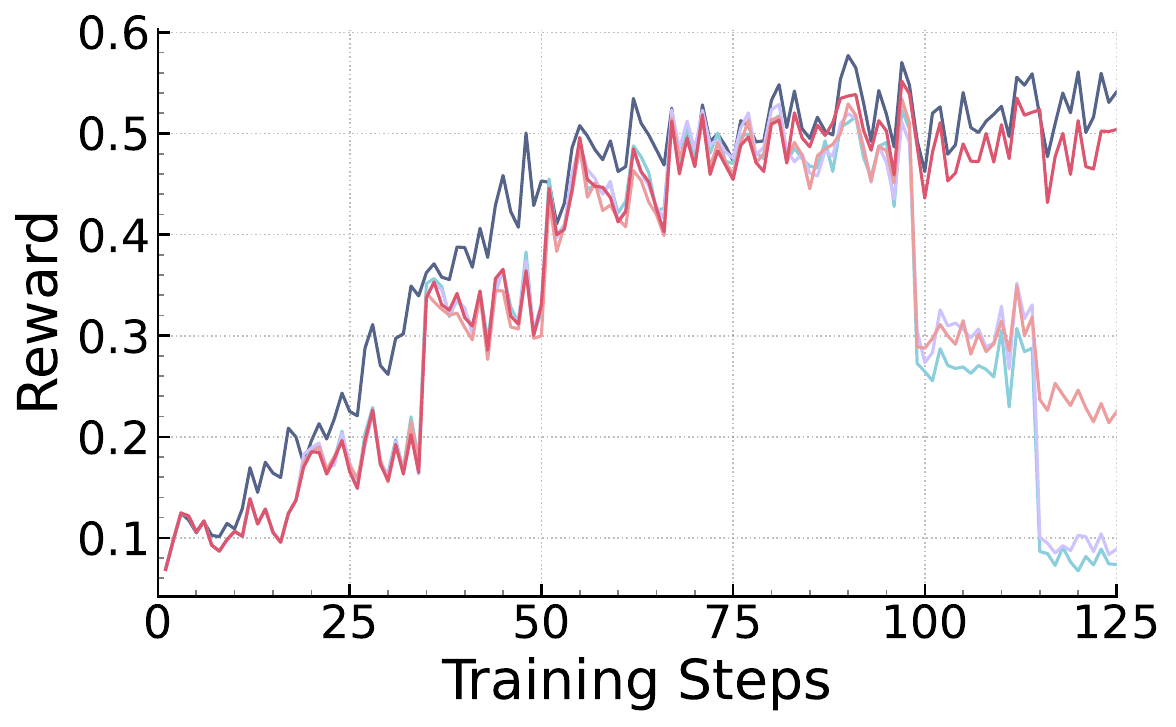}
    \caption{\small Qwen3-4B}
    \label{fig:dynamics-qwen-4b-reward}
  \end{subfigure}\hfill
  \begin{subfigure}[t]{0.24\textwidth}
    \centering
    \includegraphics[width=\linewidth]{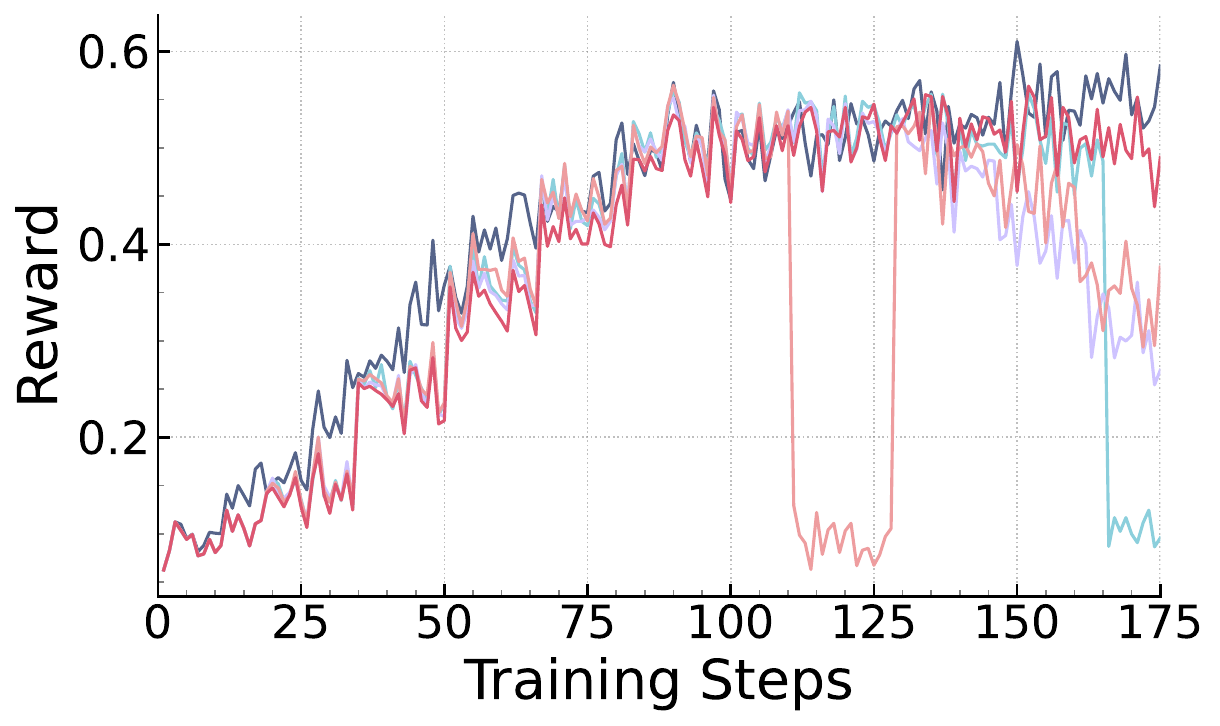}
    \caption{\small Qwen3-8B}
    \label{fig:dynamics-qwen-8b-reward}
  \end{subfigure}\hfill
  
  \vspace{0.35em}

  % =========================
  % Row 2: Accuracy (4 models)
  % =========================
  \begin{subfigure}[t]{0.24\textwidth}
    \centering
    \includegraphics[width=\linewidth]{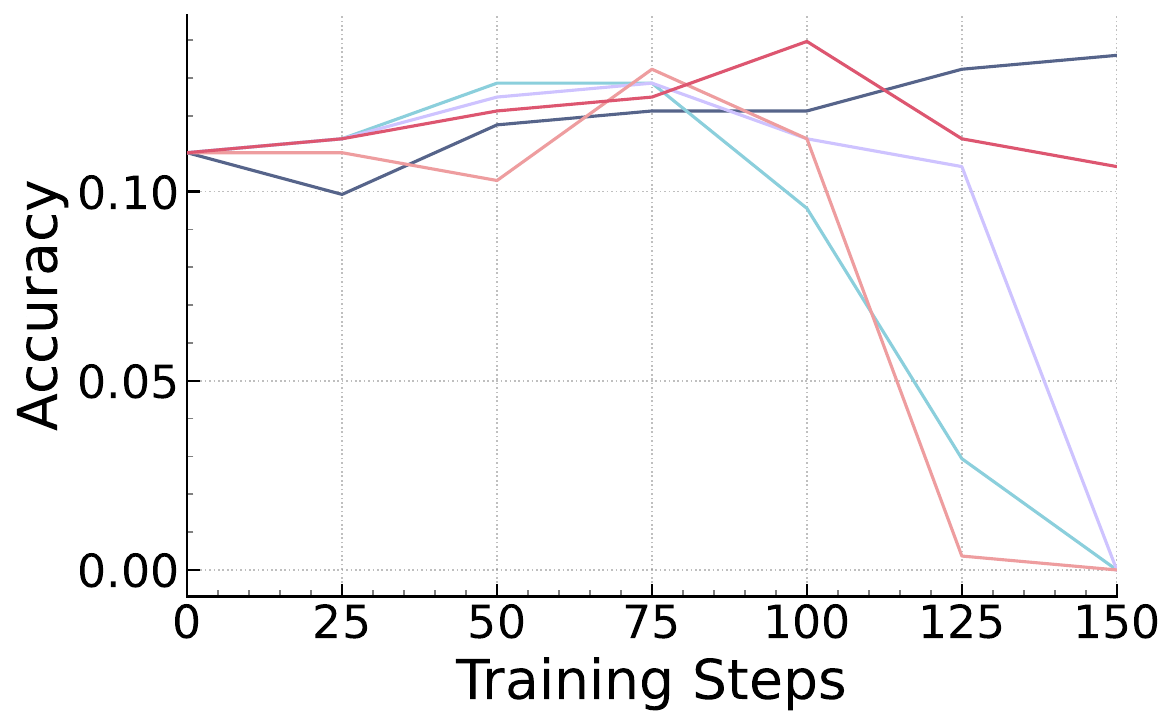}
    \caption{\small Llama-3.2-3B-Inst.}
    \label{fig:dynamics-llama-3p2-3b-acc}
  \end{subfigure}
  \begin{subfigure}[t]{0.24\textwidth}
    \centering
    \includegraphics[width=\linewidth]{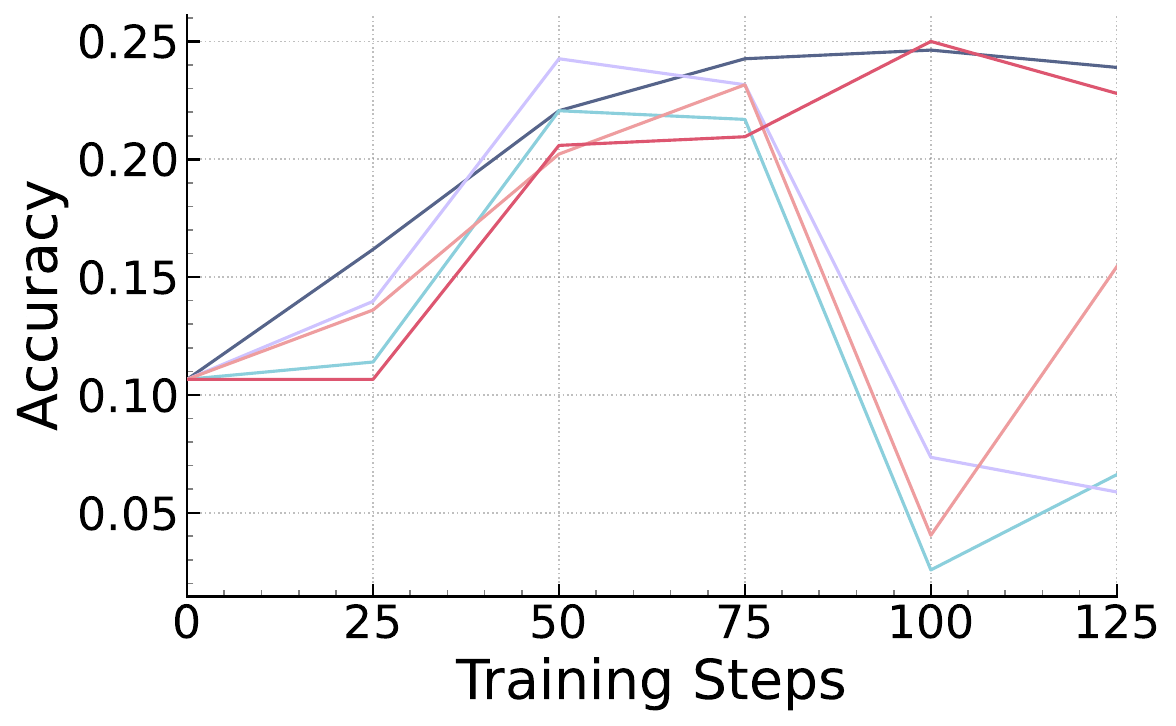}
    \caption{\small Qwen3-1.7B}
    \label{fig:dynamics-qwen-1p7b-acc}
  \end{subfigure}\hfill
  \begin{subfigure}[t]{0.24\textwidth}
    \centering
    \includegraphics[width=\linewidth]{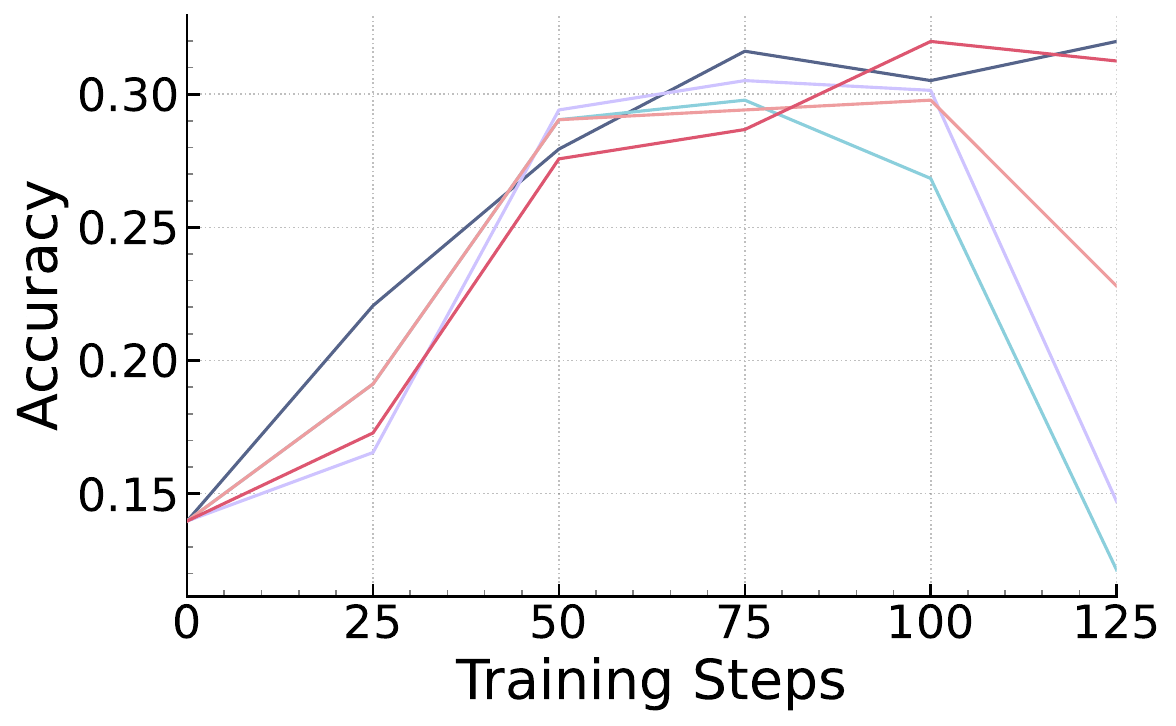}
    \caption{\small Qwen3-4B}
    \label{fig:dynamics-qwen-4b-acc}
  \end{subfigure}\hfill
  \begin{subfigure}[t]{0.24\textwidth}
    \centering
    \includegraphics[width=\linewidth]{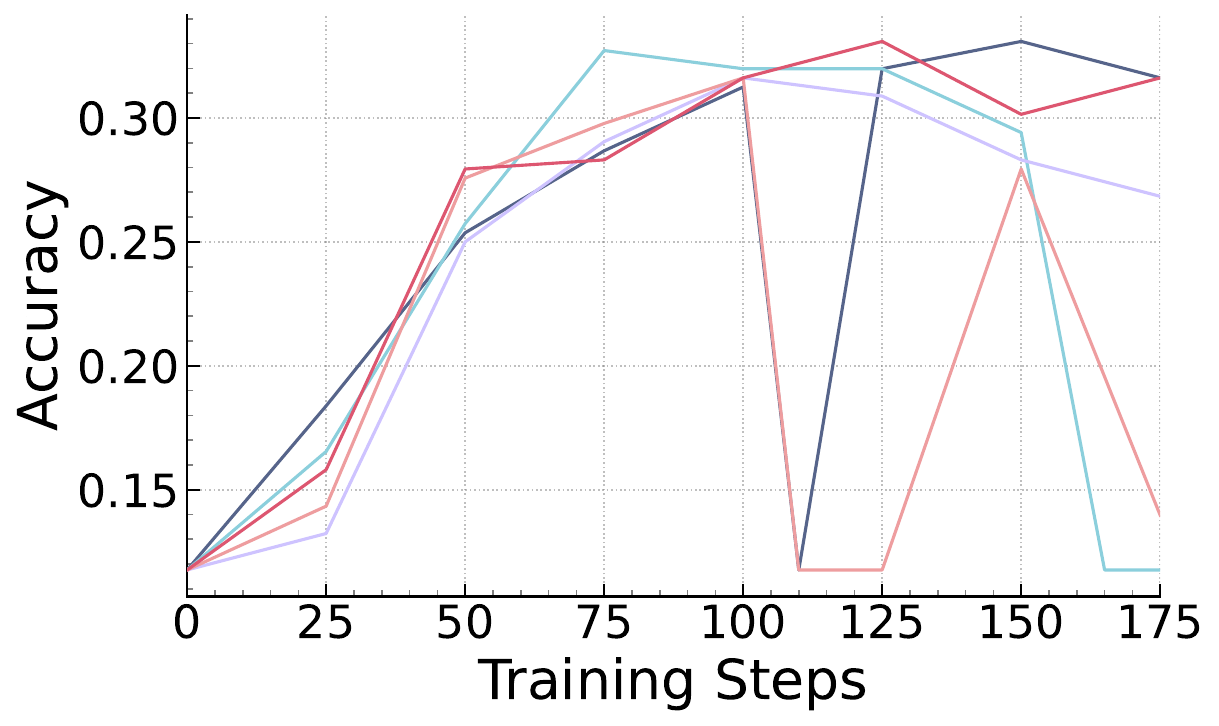}
    \caption{\small Qwen3-8B}
    \label{fig:dynamics-qwen-8b-acc}
  \end{subfigure}\hfill

  \caption{%
    \textbf{Cross-scale learning curves under large staleness.}
    Panels (a--d) report training reward and panels (e--h) report MinervaMath validation accuracy for Qwen3-1.7B/4B/8B and Llama-3.2-3B-Instruct.
    Curves compare synchronized GRPO, stale-rollout GRPO, M2PO, BAPO, and \textsc{GAC}.
  }
  
  \label{fig:dynamics-qwen-llama-2x4}
  \vspace{-3mm}
\end{figure*}

\paragraph{Experimental setup.}
We extensively evaluate \textsc{GAC} on seven mathematical reasoning benchmarks:
AIME24/25~\citep{aopsAIME}, AMC23/24~\citep{aopsAMC}, Math500~\citep{hendrycks2021math_arxiv}, OlympiadBench~\citep{he2024olympiadbench}, and MinervaMath~\citep{lewkowycz2022minerva}.
Experiments use widely adopted Qwen3 model family at three parameter scales (1.7B/4B/8B)~\citep{yang2025qwen3}, and Llama3.2-3B-Instruct~\citep{grattafiori2024llama3} to verify robustness across different architectures. 
For training, we adopt the mathematical reasoning dataset from DeepScaleR~\citep{luo2025deepscaler}.
For all models, the context length is set to 16K tokens.
For evaluation, we report Pass@1(avg@16) on AIME24/25 and Pass@1(avg@4) on the remaining benchmarks.
We vary the staleness level over $s \in \{4,8,16,32\}$ to assess performance under different degrees of asynchrony.
Unless otherwise stated, \textsc{GAC} uses a single unified threshold configuration $(c_{\text{low}}, c_{\text{high}})$ across all models, scales, and benchmarks without per-setting tuning; threshold sensitivity analysis is deferred to Appendix~\ref{app:threshold}.
We compare against three off-policy baselines: \textbf{GRPO}~\citep{shao2024deepseekmath}, standard GRPO on stale rollouts; \textbf{M2PO}~\citep{zheng2025m2po}, a trust-region method for large staleness; and \textbf{BAPO}~\citep{xi2025bapo}, an adaptive clipping approach for off-policy RL. Fully synchronized \textbf{GRPO ($s{=}0$)} serves as the on-policy reference.

% Unless otherwise stated, \textsc{GAC} adopts a single, fixed threshold configuration,
% $(c_{\text{low}}, c_{\text{high}})=(0.05,0.3)$, across all model architectures, scales, and benchmarks.
% These thresholds are not tuned per setting: $c_{\text{low}}$ is anchored to the empirical on-policy alignment range observed under synchronized GRPO training, while $c_{\text{high}}$ is set as a conservative boundary for severe alignment events.
% Appendix~\ref{app:threshold} provides the detailed rationale and threshold-sensitivity ablations, showing that \textsc{GAC} is robust to nearby threshold choices.

% \paragraph{Baselines.}
% We compare GAC against three off-policy baselines designed for stale-rollout optimization:
% \textbf{GRPO}~\citep{shao2024deepseekmath}, standard GRPO applied directly to stale rollouts;
% \textbf{M2PO}~\citep{zheng2025m2po}, a trust-region method designed to handle large staleness;
% and \textbf{BAPO}~\citep{xi2025bapo}, an adaptive clipping approach for off-policy RL.
% We also report \textbf{GRPO ($s{=}0$)}, fully synchronized training with fresh rollouts, as the on-policy reference.

\begin{figure}[!tbp]
  \centering
  \begin{subfigure}[t]{0.40\linewidth}
    \centering
    \includegraphics[width=\linewidth]{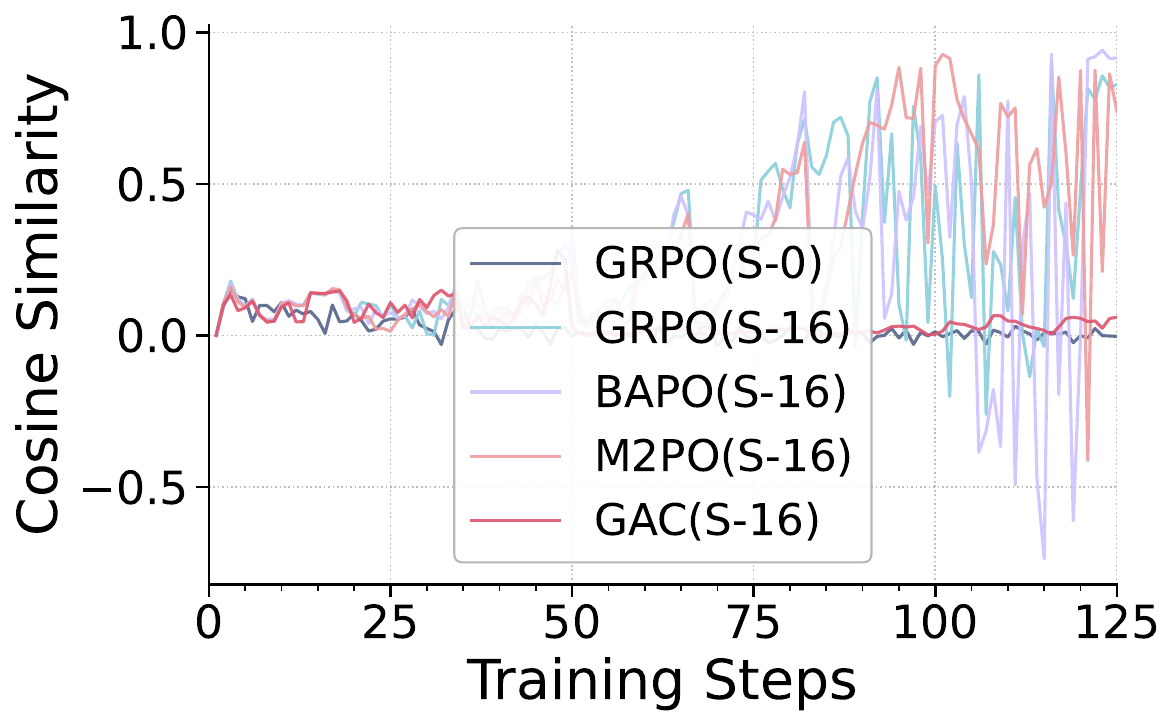}
    \caption{\small Qwen3-4B: $c_t$}
    \label{fig:ct-4b}
  \end{subfigure}
  \hspace{0.04\linewidth}
  \begin{subfigure}[t]{0.40\linewidth}
    \centering
    \includegraphics[width=\linewidth]{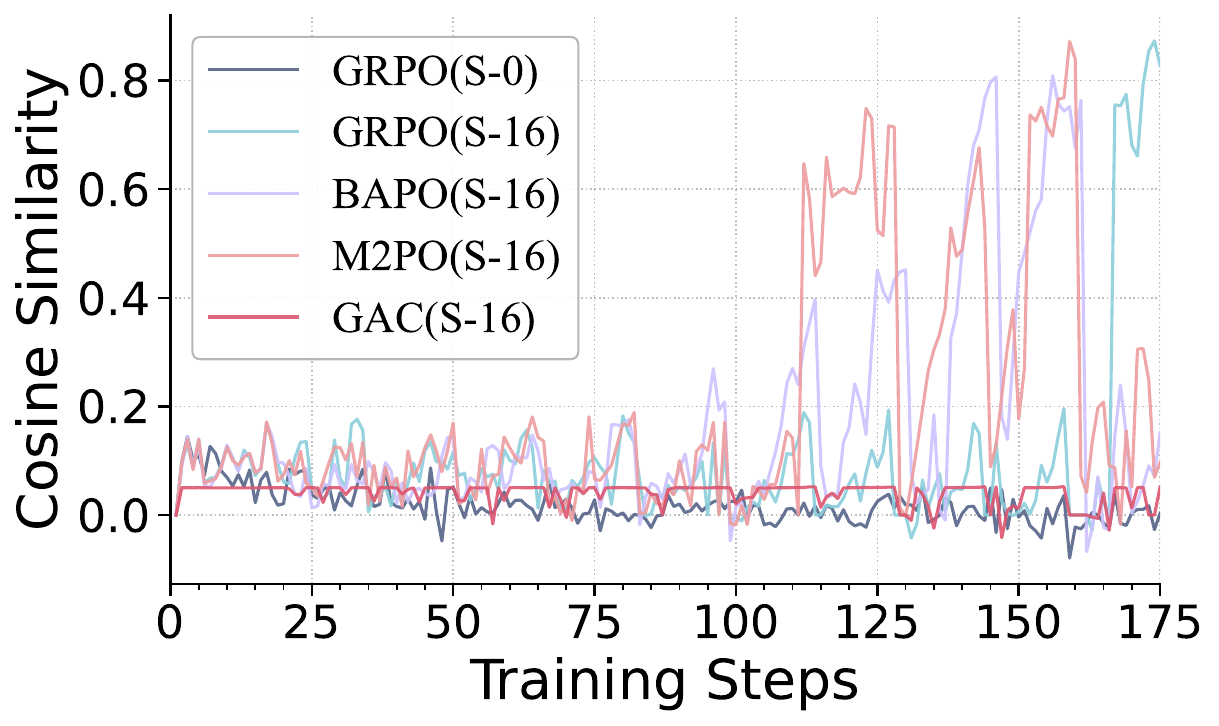}
    \caption{\small Qwen3-8B: $c_t$}
    \label{fig:ct-8b}
  \end{subfigure}

  \caption{
    \textbf{Gradient-alignment dynamics under staleness.}
    Consecutive-gradient cosine similarity $c_t$ for Qwen3-4B/8B, comparing synchronized GRPO, stale-rollout GRPO, M2PO, BAPO and GAC.
  }
  \label{fig:ct-4b-8b-1x2}
  \vspace{-6mm}
\end{figure}

\begin{table*}[t]
  \centering
  \caption{\textbf{Accuracy under Staleness.}
Accuracy (\%) on seven mathematical reasoning benchmarks for Qwen3 models (1.7B/4B/8B) and Llama-3.2-3B-Instruct.
We report synchronized on-policy GRPO (\texttt{GRPO-Sync}, $s{=}0$) as a \emph{reference} (shaded rows).
All other methods are evaluated under stale-rollout training with $s{=}16$.
\textbf{Avg.} denotes the mean accuracy across benchmarks.
\textbf{Bold} marks the best average accuracy among stale-rollout methods for each model.
The last column reports $\Delta_{\text{Avg}}$, computed as the improvement of \textsc{GAC} over the strongest stale-rollout baseline among \{GRPO, M2PO, BAPO\} for the same model.}

  \label{tab:main-results}
  \small
  \setlength{\tabcolsep}{5pt}
  \renewcommand{\arraystretch}{1.08}
  \newcommand{\dash}{\textcolor{gray}{--}}

  % columns: Method | 7 benchmarks | Avg | ΔAvg
  \begin{tabular}{l|ccccccc|c|c}
    \hline\hline
    \textbf{Method} &
    \textbf{AIME24} & \textbf{AIME25} &
    \textbf{AMC23}  & \textbf{AMC24} &
    \textbf{Math500} & \textbf{Miner.} & \textbf{Olymp.} &
    \textbf{Avg.} & $\boldsymbol{\Delta_{\textbf{Avg}}}$ \\
    \hline\hline

    \multicolumn{10}{c}{\textit{Llama-3.2-3B-Instruct}} \\
    \hline
    \rowcolor{gray!12} GRPO-Sync & 17.2 & 3.9 & 35.2 & 18.9 & 56.5 & 22.3 & 20.5 & 24.9 & \dash \\
    GRPO & 10.8 & 3.0 & 30.4 & 13.8 & 55.4 & 19.1 & 17.6 & 21.4 & \dash \\
    M2PO & 13.3 & 3.4 & 32.5 & 15.6 & 55.7 & 19.9 & 18.9 & 22.8 & \dash \\
    BAPO &  8.6 & 3.3 & 33.3 & 15.2 & 55.5 & 19.2 & 18.6 & 22.0 & \dash \\
    \rowcolor{blue!8} \textbf{GAC} & 17.0 & 3.7 & 36.1 & 19.0 & 56.2 & 21.6 & 20.9 & \textbf{24.9} & \textbf{+2.2} \\
    \hline

    \multicolumn{10}{c}{\textit{Qwen3-1.7B}} \\
    \hline
    \rowcolor{gray!12} GRPO-Sync & 17.2 & 17.0 & 55.5 & 54.3 & 76.1 & 25.1 & 27.1 & 38.9 & \dash \\
    GRPO & 11.1 & 12.5 & 47.8 & 44.4 & 69.2 & 22.0 & 23.4 & 32.9 & \dash \\
    M2PO & 13.9 & 13.8 & 45.3 & 41.2 & 73.3 & 23.1 & 24.5 & 33.6 & \dash \\
    BAPO & 16.1 & 13.6 & 50.2 & 40.0 & 71.0 & 24.2 & 23.6 & 34.1 & \dash \\
    \rowcolor{blue!8} \textbf{GAC} & 17.8 & 17.0 & 54.3 & 52.7 & 74.9 & 25.0 & 26.9 & \textbf{38.4} & \textbf{+4.3} \\
    \hline

    \multicolumn{10}{c}{\textit{Qwen3-4B}} \\
    \hline
    \rowcolor{gray!12} GRPO-Sync & 32.1 & 31.8 & 78.1 & 69.9 & 84.6 & 39.2 & 45.3 & 54.4 & \dash \\
    GRPO & 24.3 & 24.0 & 71.2 & 61.8 & 77.9 & 32.2 & 37.9 & 47.0 & \dash \\
    M2PO & 29.4 & 25.7 & 73.9 & 63.2 & 79.0 & 35.3 & 41.2 & 49.7 & \dash \\
    BAPO & 28.8 & 26.9 & 75.3 & 60.1 & 80.1 & 34.5 & 42.3 & 49.7 & \dash \\
    \rowcolor{blue!8} \textbf{GAC} & 31.9 & 31.7 & 81.1 & 66.5 & 84.2 & 37.9 & 44.7 & \textbf{54.0} & \textbf{+4.3} \\
    \hline

    \multicolumn{10}{c}{\textit{Qwen3-8B}} \\
    \hline
    \rowcolor{gray!12} GRPO-Sync & 29.2 & 29.0 & 80.4 & 66.8 & 84.1 & 42.3 & 44.5 & 53.8 & \dash \\
    GRPO & 24.0 & 23.7 & 73.8 & 62.4 & 82.5 & 34.8 & 35.0 & 48.0 & \dash \\
    M2PO & 25.7 & 23.4 & 75.9 & 64.0 & 82.7 & 37.1 & 39.9 & 49.8 & \dash \\
    BAPO & 26.3 & 22.4 & 74.7 & 66.9 & 82.0 & 36.2 & 36.5 & 49.3 & \dash \\
    \rowcolor{blue!8} \textbf{GAC} & 31.9 & 28.8 & 78.9 & 66.5 & 83.9 & 42.8 & 43.2 & \textbf{53.7} & \textbf{+3.9} \\
    \hline\hline
  \end{tabular}
  \vspace{-2mm}
\end{table*}

\subsection{Training Dynamics Under Staleness}
\label{sec:dynamics}

We evaluate the efficacy of GAC under large staleness ($s=16$, which saturates throughput gains in most asynchronous RL systems) by analyzing both convergence behavior and the underlying gradient-alignment mechanism.

\paragraph{Convergence behavior.}
Figure~\ref{fig:dynamics-qwen-llama-2x4} presents training reward and validation accuracy (from MinervaMath ~\citep{lewkowycz2022minerva} for Qwen3-1.7B/4B/8B and Llama-3.2-3B-Instruct.
All off-policy baselines exhibit unstable convergence across model scales: reward and accuracy undergo severe fluctuations, abrupt drops, and in extreme cases, complete collapse.
In contrast, \textsc{GAC} maintains smooth monotonic convergence throughout training, achieving stability comparable to synchronized GRPO.
Critically, \textsc{GAC} achieves this stability without sacrificing sample efficiency, converging at rates comparable to synchronized training without requiring extended optimization and reaching similar final performance.
The consistent behavior across the 1.7B--8B range and both model families demonstrates that explicit alignment control resolves the fundamental instability of asynchronous GRPO.

\paragraph{Gradient alignment regulation.}
Figure~\ref{fig:ct-4b-8b-1x2} reports the consecutive-gradient cosine similarity $c_t$ for Qwen3-4B and Qwen3-8B, revealing the mechanism behind \textsc{GAC}'s stabilization.
Synchronized training (GRPO, $s{=}0$) and \textsc{GAC} both keep $|c_t|$ tightly concentrated near zero with negligible variance throughout optimization, indicating near-orthogonal consecutive gradients.
In contrast, the off-policy baselines—stale-rollout GRPO, M2PO and BAPO—exhibit a sustained rise in $|c_t|$ that escalates to high magnitudes with severe oscillations.
\textsc{GAC} restores gradient-alignment dynamics to the same low-alignment regime observed in synchronized training, consistent with its intended mechanism of regulating update geometry under staleness.

\begin{wrapfigure}{r}{0.4\textwidth}
  \centering
    \includegraphics[width=0.39\textwidth]{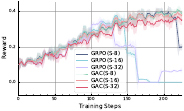}
  \vspace{-0.1cm}
  \caption{
    \textbf{Robustness across staleness levels.}
    Training reward trajectories under $s\in\{8,16,32\}$.
    \textsc{GAC} remains stable as staleness increases, while stale-rollout GRPO progressively degrades.
  }
  \label{fig:staleness-sweep}
  \vspace{-0.1cm}
\end{wrapfigure}
\subsection{Model Performance}
\label{sec:main-results}
Table~\ref{tab:main-results} summarizes accuracy on seven reasoning benchmarks under a large staleness setting ($s{=}16$), with synchronized GRPO ($s{=}0$) reported as an on-policy reference.
Across model families and scales, stale-rollout GRPO and strong off-policy baselines (M2PO, BAPO) exhibit substantial accuracy degradation under staleness, consistently underperforming the on-policy reference.
This persistent gap underscores the fundamental challenge of maintaining stable optimization under asynchrony.
In contrast, \textsc{GAC} achieves the highest average accuracy among stale-rollout methods and largely closes the gap to synchronized training.
On Qwen3-8B, \textsc{GAC} achieves 53.7\% compared to 53.8\% for the on-policy reference, while off-policy baselines fall 4--5 points below.
The gap is most pronounced on challenging benchmarks; for instance, on AIME24/25 with Qwen3-4B, \textsc{GAC} attains 31.9/31.7, essentially matching synchronized training (32.1/31.8), while off-policy baselines degrade substantially.

% \begin{wrapfigure}{r}{0.40\linewidth}
%   \vspace{2mm}
%   \centering
%   \includegraphics[width=1\linewidth]{figure/experiment/stale.pdf}
%   \caption{
%     \textbf{Robustness across staleness levels.}
%     Training reward trajectories under $s\in\{8,16,32\}$.
%     \textsc{GAC} remains stable as staleness increases, while stale-rollout GRPO progressively degrades.
%   }
%   \label{fig:staleness-sweep}
%   % \vspace{0.7em}
% \end{wrapfigure}

\paragraph{Robustness to varying staleness.}
\label{sec:ablations}
We extend our analysis to $s\in\{4,8,16,32\}$ to assess generalization across asynchrony regimes. Figure~\ref{fig:staleness-sweep} shows that \textsc{GAC} maintains stable convergence across all staleness levels, including $s{=}32$ where stale-rollout GRPO collapses entirely. In contrast to the progressive degradation observed in Section~\ref{subsec:problem-severity}, \textsc{GAC} achieves stable performance across the entire staleness range.
\vspace{-4pt}

\section{Related Work}
Due to space limitations, we focus on the most directly related work in the main text and defer a more detailed discussion to Appendix~\ref{sec:related-work-extended}.

As reinforcement learning workloads scale, end-to-end training is increasingly constrained by heterogeneous pipeline stages, where rollout generation and reward evaluation often dominate wall-clock time.
A widely adopted systems solution is to use disaggregated actor--learner architectures that decouple trajectory collection from optimization, thereby improving throughput and mitigating straggler effects.
Early asynchronous methods such as A3C~\cite{mnih2016a3c} demonstrated that relaxing strict synchronization can substantially improve scalability, while IMPALA~\cite{espeholt2018impala} introduced importance-corrected actor--learner updates (via V-trace) to control off-policy drift at scale.
Subsequent systems refined this paradigm through more aggressive data reuse and communication-efficient designs, including Ape-X~\cite{horgan2018distributed} and SEED RL~\cite{espeholt2020seedrl}.
Collectively, these works establish asynchrony and staleness as practical ingredients for high-throughput RL, while also highlighting that delayed gradients and stale rollouts can fundamentally alter optimization behavior and introduce additional instability.

Motivated by similar efficiency bottlenecks in large language model post-training, recent work revisits asynchronous execution for RLHF and RLVR under expensive inference-time rollouts.
Asynchronous RLHF formulations explicitly exploit off-policy data reuse to improve throughput~\cite{noukhovitch2024asynchronousrlhf}, while large-scale LLM-RL infrastructures such as AReaL~\cite{fu2025areal} and AsyncFlow~\cite{han2025asyncflow} support decoupled rollout and optimization for scalable reasoning post-training.
However, this decoupling naturally introduces stale-rollout gradients, where updates are computed using trajectories generated by lagged behavior policies.
Under large staleness, this effect amplifies distribution mismatch and destabilizes GRPO-style optimization.
Recent off-policy stabilization methods for LLM-RL, such as M2PO~\cite{zheng2025m2po} and BAPO~\cite{xi2025bapo}, primarily address this issue by strengthening GRPO-style trust-region surrogates through adaptive clipping or likelihood-ratio shaping.

In contrast, our method operates at the optimizer interface and is complementary to these approaches.
Rather than modifying surrogate objectives or correcting distributional shift, we directly regulate the geometry of the aggregated policy update by suppressing progress along stale-aligned directions.
This provides a lightweight, system-agnostic stabilization mechanism that targets the alignment-driven instability induced by asynchronous execution.
\vspace{-4pt}
\section{Conclusion}
\label{sec:conclusion}
\vspace{-4pt}
This paper studies asynchronous GRPO from a training-dynamics perspective and demonstrates that asynchrony can fundamentally alter training dynamics, leading to collapse.
Through both empirical evidence and theoretical analysis, we identify persistently elevated consecutive-gradient cosine similarity as a reliable precursor to training collapse, capturing the accumulation of stale-aligned updates.
Motivated by this insight, we propose \textsc{GAC}, a lightweight gradient-space controller that stabilizes asynchronous GRPO by explicitly regulating progress along stale-aligned directions while preserving informative residual updates.
\textsc{GAC} operates at the optimizer interface, is system- and algorithm-agnostic, and incurs only constant per-step overhead.
Our results demonstrate that \textsc{GAC} effectively stabilizes asynchronous GRPO and restores robust training dynamics under staleness.
Our evaluation focuses on RLVR-style mathematical reasoning tasks; extending GAC to broader settings such as instruction-following and agentic tool-use remains future work.
Furthermore, while GAC closes the performance gap to synchronized training, it is designed to stabilize asynchronous optimization rather than surpass the on-policy reference, and its effectiveness under more aggressive staleness or heterogeneous pipeline configurations warrants further investigation.

\clearpage
\newpage
% \section*{Impact Statement}
% This paper presents work whose goal is to advance the training of reinforcement learning. There are many potential societal consequences of our work, none of which we feel must be specifically highlighted here.

\bibliographystyle{unsrt}
\bibliography{reference}

%%%%%%%%%%%%%%%%%%%%%%%%%%%%%%%%%%%%%%%%%%%%%%%%%%%%%%%%%%%%%%%%%%%%%%%%%%%%%%%
%%%%%%%%%%%%%%%%%%%%%%%%%%%%%%%%%%%%%%%%%%%%%%%%%%%%%%%%%%%%%%%%%%%%%%%%%%%%%%%
% APPENDIX
%%%%%%%%%%%%%%%%%%%%%%%%%%%%%%%%%%%%%%%%%%%%%%%%%%%%%%%%%%%%%%%%%%%%%%%%%%%%%%%
%%%%%%%%%%%%%%%%%%%%%%%%%%%%%%%%%%%%%%%%%%%%%%%%%%%%%%%%%%%%%%%%%%%%%%%%%%%%%%%

\clearpage
\appendix
\section{Additional Details and Discussion}\label{appdx:additional}

\subsection{Computing Consecutive-Gradient Cosine Similarity}
\label{app:ct-computation}
We provide implementation details for the consecutive-gradient cosine similarity $c_t$ in Eq.~\eqref{eq:cos-sim}.
At update step $t$, we define $\hat{g}_t$ as the aggregated policy gradient \emph{estimate} after completing backpropagation and gradient accumulation for the GRPO objective (using the rollout data available at step $t$, which may be stale).
We conceptually treat $\hat{g}_t$ as the flattened concatenation of all trainable parameter gradients.
The metric $c_t$ is scale-invariant and serves as a geometric diagnostic of successive update directions.

\paragraph{Efficient computation under data parallelism.}
Under distributed training with gradient sharding (e.g., FSDP, DeepSpeed ZeRO), we compute $c_t$ without materializing full gradients.
Let $\hat{g}_t^{(r)}$ and $\hat{g}_{t-1}^{(r)}$ denote gradient shards on rank $r$.
Each rank computes three local dot products:
\begin{equation}
\langle \hat{g}_t^{(r)}, \hat{g}_{t-1}^{(r)}\rangle, \quad
\langle \hat{g}_t^{(r)}, \hat{g}_t^{(r)}\rangle, \quad
\langle \hat{g}_{t-1}^{(r)}, \hat{g}_{t-1}^{(r)}\rangle,
\end{equation}
followed by a single all-reduce to obtain global quantities
\begin{equation}
\langle \hat{g}_t, \hat{g}_{t-1}\rangle, \quad
\|\hat{g}_t\|_2^2, \quad
\|\hat{g}_{t-1}\|_2^2.
\end{equation}
We then compute
\begin{equation}
c_t = \frac{\langle \hat{g}_t, \hat{g}_{t-1}\rangle}{\sqrt{\|\hat{g}_t\|_2^2 \cdot \|\hat{g}_{t-1}\|_2^2} + \epsilon},
\end{equation}
where $\epsilon=10^{-8}$ ensures numerical stability.

\paragraph{Evaluation protocol.}
Unless otherwise stated, $c_t$ is computed on raw aggregated gradient estimates $\hat{g}_t$ before optimizer transformations or \textsc{GAC} intervention.

\subsection{Computational and Memory Overhead}
\label{app:gac-overhead}

\paragraph{Computational overhead.} 
\textsc{GAC} introduces only lightweight overhead beyond standard gradient computation.
Per update, it evaluates the consecutive-gradient cosine similarity $c_t$ using a constant number of inner products on sharded gradients and a constant-size collective reduction to aggregate the required statistics across workers.
Under distributed training with gradient sharding (e.g., FSDP or DeepSpeed ZeRO), this adds negligible communication volume relative to routine training collectives, with overhead primarily arising from an additional synchronization point.

When $|c_t|$ falls in the projection regime ($c_{\text{low}} < |c_t| < c_{\text{high}}$), \textsc{GAC} applies an adaptive rank-one correction to the current gradient:
\begin{equation}
\hat{g}_t' \;=\; \beta\,\hat{g}_t \;+\; (\alpha-\beta)\,\langle \hat{g}_t, u_{t-1}\rangle u_{t-1},
\end{equation}
where $\alpha = c_{\text{low}}/|c_t|$ is computed dynamically and we use $\beta = 1$.
This intervention consists of a small number of elementwise linear operations (a scale and an addition) and therefore incurs $O(d_{\text{local}})$ memory-bandwidth work per update, where $d_{\text{local}}$ is the number of gradient elements on the local rank.
The projection is applied in-place to parameter gradients and does not require additional forward/backward passes.

\paragraph{Memory footprint.}
\textsc{GAC} maintains a single additional gradient snapshot for the previous update direction $\hat{g}_{t-1}$, resulting in an $O(d_{\text{local}})$ persistent memory overhead per rank.
Any temporary buffers used to flatten gradients for similarity computation are short-lived and do not affect the asymptotic memory complexity.

\subsection{Pseudocode for \textsc{GAC}}
\label{appdx:gac-pseudocode}

Algorithm~\ref{alg:gac} provides the complete pseudocode of \textsc{GAC}. 
At each update, \textsc{GAC} computes the consecutive-gradient cosine similarity $c_t$ and selects among three behaviors: a standard gradient step in the low-alignment regime, an adaptive projection step that suppresses the component aligned with the previous gradient, or an update skip when alignment exceeds the high-threshold safeguard.

\begin{algorithm}[t]
\caption{Gradient Alignment Control (GAC)}
\label{alg:gac}
\begin{algorithmic}[1]
\STATE \textbf{Input:} thresholds $c_{\text{low}}$, $c_{\text{high}}$, learning rate $\eta$
\STATE \textbf{Initialize:} $\theta_0$, $\hat{g}_0 \leftarrow 0$
\FOR{$t = 1, 2, \ldots, T$}
    \STATE Compute gradient $\hat{g}_t$ from stale rollouts under $\pi_{\theta_{t-\tau}}$
    \STATE $c_t \leftarrow \langle \hat{g}_t, \hat{g}_{t-1}\rangle / (\|\hat{g}_t\|_2\,\|\hat{g}_{t-1}\|_2)$
    \IF{$|c_t| \le c_{\text{low}}$}
        \STATE $\theta_{t} \leftarrow \theta_{t-1} + \eta\,\hat{g}_t$
    \ELSIF{$|c_t| \ge c_{\text{high}}$}
        \STATE $\theta_{t} \leftarrow \theta_{t-1}$
    \ELSE
        \STATE $u \leftarrow \hat{g}_{t-1}/\|\hat{g}_{t-1}\|_2$, \quad $\alpha \leftarrow c_{\text{low}}/|c_t|$
        \STATE $g_{\parallel} \leftarrow \langle \hat{g}_t, u\rangle u$, \quad $g_{\perp} \leftarrow \hat{g}_t - g_{\parallel}$
        \STATE $\hat{g}_t' \leftarrow \alpha\, g_{\parallel} + g_{\perp}$
        \STATE $\theta_{t} \leftarrow \theta_{t-1} + \eta\,\hat{g}_t'$
    \ENDIF
\ENDFOR
\end{algorithmic}
\end{algorithm}
\section{Extended Related Work}
\label{sec:related-work-extended}

\paragraph{RLVR and GRPO Variants.}
Reinforcement learning has become a standard paradigm for post-training large language models toward complex objectives.
Algorithmically, modern RL for LLMs largely builds upon trust-region policy optimization, from TRPO~\cite{schulman2015trpo} to PPO~\cite{schulman2017ppo}, which stabilizes updates via constrained or clipped surrogate objectives.
These foundations enabled RLHF to achieve strong empirical gains in summarization, instruction following, and safety alignment~\cite{ziegler2019finetuning,ouyang2022training,bai2022training}.
More recently, reinforcement learning with verifiable rewards (RLVR) has gained prominence in reasoning-centric post-training, substituting learned reward models with programmatic verifiers that yield sparse, binary feedback~\citep{lightman2023lets,uesato2022solving}.
This paradigm shift necessitates optimization methods robust to high-variance, outcome-level signals.
GRPO~\citep{shao2024deepseekmath,deepseek2025r1} has emerged as a canonical RLVR algorithm, inspiring a growing family of variants designed to enhance stability and credit assignment under verifier sparsity. Notable examples include DAPO~\citep{yu2025dapo}, which decouples clipping design from sampling allocation to prioritize informative rollouts; GSPO~\citep{zheng2024gspo}, which applies sequence-level rather than token-level trust-region control; SAPO~\citep{gao2025sapo}, which adapts clipping thresholds dynamically to improve stability-exploration tradeoffs; and TR-GRPO~\citep{le2025trgrpo}, which introduces token-aware modifications for robust credit assignment. Beyond surrogate-level redesigns, Light-R1~\citep{wen2025lightr1} demonstrates a practical post-training pipeline for long-CoT reasoning in which GRPO is employed as the RL stage on long-CoT models, further improving mathematical reasoning performance. 
While these methods primarily improve GRPO via surrogate-level redesigns,
we take an optimizer-level perspective that stabilizes the policy update by directly controlling its geometry.

\paragraph{Asynchronous RL Training Frameworks and Off-policy RL Algorithms.}
As RL workloads scale, end-to-end training is increasingly constrained by heterogeneous pipeline stages, where rollout generation and reward evaluation often dominate wall-clock time.
A common systems solution is to adopt disaggregated actor--learner architectures that decouple trajectory collection from optimization, thereby improving throughput and mitigating straggler effects.
Early asynchronous methods such as A3C~\cite{mnih2016a3c} demonstrated that relaxing strict synchronization can substantially improve scalability, while IMPALA~\cite{espeholt2018impala} introduced an importance-corrected actor--learner design (via V-trace) to control off-policy drift at scale.
Subsequent systems further refined the same principle through more aggressive data reuse and communication efficiency, including replay-centric distributed training in Ape-X~\cite{horgan2018distributed} and accelerated centralized inference in SEED RL~\cite{espeholt2020seedrl}.
Collectively, these works establish asynchrony and staleness as practical ingredients for high-throughput RL, while also underscoring that delayed gradients and stale rollouts can introduce nontrivial off-policy effects that may require additional stabilization when instantiated in GRPO-style LLM post-training.

Motivated by similar efficiency bottlenecks in LLM post-training, recent work revisits asynchronous execution for RLHF and RLVR under expensive inference-time rollouts.
For example, asynchronous RLHF formulations explicitly exploit off-policy data reuse to improve throughput~\cite{noukhovitch2024asynchronousrlhf},
while large-scale LLM-RL infrastructures such as AReaL~\cite{fu2025areal} and AsyncFlow~\cite{han2025asyncflow}
support decoupled rollout--train execution for scalable reasoning post-training. However, this decoupling naturally introduces \emph{stale-rollout gradients}, where updates are computed on trajectories generated by lagged behavior policies, amplifying distribution mismatch and destabilizing GRPO-style optimization under large staleness.
To address this challenge, recent off-policy stabilization methods for LLM-RL, such as M2PO~\cite{zheng2025m2po} and BAPO~\cite{xi2025bapo},
primarily reinforce GRPO-style trust-region surrogates through adaptive clipping and likelihood-ratio shaping, which improves robustness under stale data.
In contrast, our method is complementary and operates at the optimizer interface:
we directly regulate the geometry of the aggregated update by suppressing progress along stale-aligned directions,
providing a lightweight, system-agnostic stabilization layer for asynchronous RL pipelines.
\section{Implementation Details}
\label{app:impl}

\begin{table}[t]
\centering
\caption{\textbf{GRPO hyperparameters.} Configuration used across all experiments.}
\label{tab:grpo_hyperparams}
\small
\setlength{\tabcolsep}{7pt}
\renewcommand{\arraystretch}{1.08}
\begin{tabular}{l c}
\toprule
\textbf{Parameter} & \textbf{Value} \\
\midrule
\multicolumn{2}{c}{\textit{Optimization}} \\
\midrule
Optimizer & AdamW \\
Learning rate ($\eta$) & $1\times 10^{-6}$ \\
AdamW $\beta$'s & $(0.9,\,0.999)$ \\
AdamW $\epsilon_{\mathrm{adam}}$ & $1\times 10^{-8}$ \\
Weight decay & $1\times 10^{-2}$ \\
Clip ratio ($\epsilon_{\mathrm{clip}}$) & 0.2 \\
Entropy coefficient & 0.001 \\
KL coefficient & 0.001 \\
KL loss type & \texttt{low\_var\_kl} \\
Gradient clipping & True \\
\midrule
\multicolumn{2}{c}{\textit{Training}} \\
\midrule
Train batch size & 256 \\
Responses per prompt ($n$) & 8 \\
Critic warmup steps & 0 \\
Evaluation frequency & 25 \\
\midrule
\multicolumn{2}{c}{\textit{Generation}} \\
\midrule
Sampling & True \\
Temperature & 0.6 \\
Top-$p$ & 0.95 \\
Max sequence length & 8k \\
\bottomrule
\end{tabular}
\end{table}

\subsection{Testbed}
\label{app:impl:testbed}
All experiments are conducted on a single 8-GPU machine equipped with NVIDIA H800-80GB GPUs, interconnected via 400~GB/s NVLink.
Our software stack uses CUDA 12.6, PyTorch 2.7.1, NCCL 2.26.2, vLLM 0.10.0, and FlashAttention 2.7.4.post1.

\subsection{Training framework}
\label{app:impl:framework}
All experiments are implemented on top of the open-source VERL~\citep{sheng2024hybridflow} framework, which provides an end-to-end RL post-training pipeline for large language models.
VERL couples rollout generation with policy optimization in a modular design, enabling both fully synchronized execution and stale-rollout (asynchronous) training within the same codebase.

\subsection{Hyperparameters}
\label{app:impl:hyperparams}

\paragraph{Common hyperparameters.}
We use a single shared GRPO training recipe across all methods and model scales to ensure controlled comparisons.
In particular, we fix the learning rate to $1\times10^{-6}$ and the train batch size to 256 in all experiments.
Unless otherwise stated, the remaining optimization and generation settings are also kept identical, including the clipping threshold, entropy/KL regularization, and decoding parameters.
Table~\ref{tab:grpo_hyperparams} summarizes the full set of common hyperparameters.

\paragraph{\textsc{GAC} thresholds.}
\label{app:threshold}
\textsc{GAC} introduces two alignment thresholds, $c_{\text{low}}$ and $c_{\text{high}}$, which partition the observed consecutive-gradient cosine similarity $c_t$ into the three operational regimes described in Section~\ref{sec:method:algorithm}.
Unless otherwise stated, we use a single default configuration across all experiments, setting $c_{\text{low}}=0.05$ and $c_{\text{high}}=0.3$ for both Qwen3 and Llama models and across scales.
For numerical stability when computing $c_t$, we use $\epsilon=10^{-8}$ in the cosine denominator.

\begin{table}[t]
\centering
\caption{\textbf{On-policy ($s{=}0$) statistics of consecutive-gradient alignment.}
We report the 90th percentile $q_{0.9}(|c_t|)$, the maximum $\max |c_t|$, and the fraction of updates with $|c_t|\le 0.05$ (computed over updates after step 50 to exclude early transients).
These values motivate anchoring $c_{\text{low}}$ to the on-policy range and setting $c_{\text{high}}$ well above it.}
\label{tab:onpolicy_ct_stats}
\small
\setlength{\tabcolsep}{7pt}
\renewcommand{\arraystretch}{1.05}
\begin{tabular}{lccc}
\toprule
\textbf{Model} & $q_{0.9}(|c_t|)$ & $\max |c_t|$ & $\Pr(|c_t|\le 0.05)$ \\
\midrule
Qwen3-1.7B & 0.0471 & 0.0924 & 0.91 \\
Llama-3.2-3B-Instruct & 0.0687 & 0.1276 & 0.83 \\
Qwen3-4B & 0.0296 & 0.0645 & 0.99 \\
Qwen3-8B & 0.0317 & 0.0625 & 0.98 \\
\bottomrule
\end{tabular}
\end{table}

\paragraph{Rationale for choosing $c_{\text{low}}$ and $c_{\text{high}}$.}
We choose the default thresholds to be interpretable and transferable across model families and scales.
The lower threshold $c_{\text{low}}$ is anchored to the empirical on-policy alignment range measured under synchronized GRPO ($s{=}0$).
Table~\ref{tab:onpolicy_ct_stats} summarizes on-policy statistics of $|c_t|$ (computed after step 50 to exclude early transients), showing consistently small inter-step alignment across models and scales.
Setting $c_{\text{low}}=0.05$ thus places typical on-policy updates within the low-alignment regime, providing a principled reference level without per-model tuning.
The upper threshold $c_{\text{high}}$ specifies a conservative boundary for extreme alignment events; we set $c_{\text{high}}=0.3$ to be well separated from the on-policy range (Table~\ref{tab:onpolicy_ct_stats}), providing a safeguard under large staleness.

\paragraph{Threshold robustness.}
Table~\ref{tab:gac_thr_ablation_acc} reports a local $3\times3$ sweep around $(c_{\text{low}},c_{\text{high}})=(0.05,0.3)$.
Performance remains robust across the grid, indicating that \textsc{GAC} is robust to threshold choice; we thus use $(0.05,0.3)$ unchanged in all main experiments.

\begin{figure*}[t]
  \centering

  % ---------- (a) curve ----------
  \begin{subfigure}[t]{0.62\textwidth} % was 0.70
    \centering
    \includegraphics[width=0.92\linewidth]{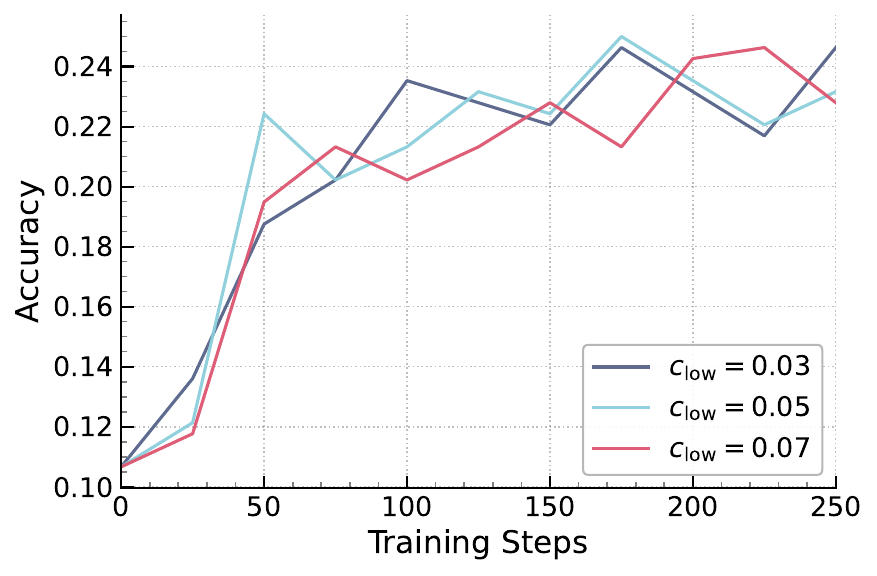} % slightly smaller
    \caption{\textbf{Fixed $c_{\text{high}}{=}0.3$.} OlympiadBench validation accuracy trajectories under $s{=}16$ for Qwen3-1.7B while sweeping $c_{\text{low}}$.}
    \label{fig:gac_thr_sweep_curve}
  \end{subfigure}

  \vspace{0.45em} % slightly tighter

  % ---------- (b) table ----------
  % \begin{subfigure}[t]{0.62\textwidth}
  % \centering
  % \small
  % \setlength{\tabcolsep}{7pt}
  % \renewcommand{\arraystretch}{1.12}
  % \begin{tabular}{@{}l|ccccc@{}}
  %   \toprule
  %   $c_{\text{low}} \ \backslash \ c_{\text{high}}$
  %     & $0.10$ & $0.20$ & $0.30$ & $0.40$ & $0.70$ \\
  %   \midrule
  %   $0.01$ & \texttt{--} & \texttt{--} & \texttt{--} & \texttt{--} & \texttt{--} \\
  %   $0.03$ & \texttt{--} & 26.8 & 26.3 & 25.5 & \texttt{--} \\
  %   $0.05$ & \texttt{--} & 26.2 & \cellcolor{gray!20}26.9 & 27.1 & \texttt{--} \\
  %   $0.07$ & \texttt{--} & 25.7 & 26.6 & 24.9 & \texttt{--} \\
  %   $0.10$ & \texttt{--} & \texttt{--} & \texttt{--} & \texttt{--} & \texttt{--} \\
  %   \bottomrule
  % \end{tabular}
  % \caption{\textbf{Grid sensitivity.} Validation accuracy on OlympiadBench (Qwen3-1.7B, $s{=}16$) across $(c_{\text{low}},c_{\text{high}})$; the default $(0.05,0.3)$ is highlighted.}
  % \label{tab:gac_thr_ablation_acc}
  % \end{subfigure}
  % ---------- (b) table ----------
\begin{subfigure}[t]{0.62\textwidth} % was 0.70
  \centering
  \small
  \setlength{\tabcolsep}{9pt} % slightly tighter
  \renewcommand{\arraystretch}{1.12} % slightly tighter
  \begin{tabular}{@{}l|ccc@{}}
    \toprule
    & $c_{\text{high}}{=}0.2$ & $c_{\text{high}}{=}0.3$ & $c_{\text{high}}{=}0.4$ \\
    \midrule
    $c_{\text{low}}{=}0.03$ & 26.8 & 26.3 & 25.5 \\
    $c_{\text{low}}{=}0.05$ & 26.2 & \cellcolor{gray!20}26.9 & 27.1 \\
    $c_{\text{low}}{=}0.07$ & 25.7 & 26.6 & 24.9 \\
    \bottomrule
  \end{tabular}
  \caption{\textbf{Grid sensitivity.} Validation accuracy on OlympiadBench (Qwen3-1.7B, $s{=}16$) across $(c_{\text{low}},c_{\text{high}})$; the default is highlighted.}
  \label{tab:gac_thr_ablation_acc}
\end{subfigure}

  \caption{\textbf{Threshold robustness of \textsc{GAC}.}
  Ablations around the default $(c_{\text{low}},c_{\text{high}}){=}(0.05,0.3)$ show that accuracy varies only mildly across nearby settings, supporting practical robustness without per-model retuning.}
  \label{fig:gac_thr_ablation_combined}
\end{figure*}

\section{On-Policy Gradient Alignment Across Model Scales}
\label{app:onpolicy-ct}
Figure~\ref{fig:app-ct-s0-vs-s16-2x2} demonstrates that synchronized GRPO maintains 
low consecutive-gradient alignment across model families and scales. Across all four models, on-policy training ($s{=}0$) keeps $|c_t|$ consistently within $[-0.05, 0.05]$ throughout optimization, indicating near-orthogonal consecutive gradients and stable descent. This narrow alignment range is remarkably robust: it holds uniformly across Qwen3-1.7B/4B/8B and Llama-3.2-3B-Instruct, motivating our choice of $c_{\text{low}}{=}0.05$ as a principled, model-agnostic threshold anchored to the empirical on-policy regime (see Appendix~\ref{app:impl}).
Asynchronous training ($s{=}16$) exhibits substantially higher and more volatile $|c_t|$, with large-magnitude oscillations observed consistently across all four models, confirming that alignment-driven instability is a general phenomenon of stale-rollout training rather than a model-specific artifact.

This gradient-geometry perspective constitutes a key contribution of our work.
We show that consecutive-gradient cosine similarity provides a direct, early indicator of training stability that precedes these conventional diagnostics, enabling proactive 
intervention before performance degradation occurs.

\begin{figure}[t]
  \centering
  \begin{subfigure}[t]{0.49\linewidth}
    \centering
    \includegraphics[width=\linewidth]{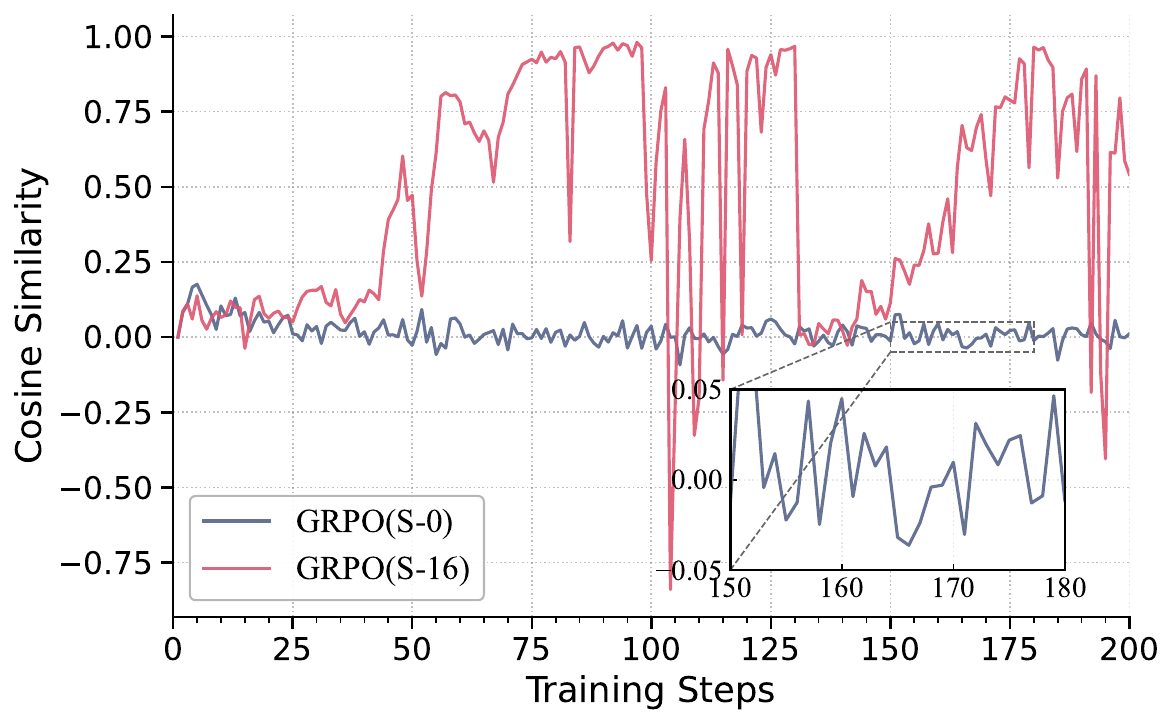}
    \caption{\small Qwen3-1.7B}
    \label{fig:app-ct-qwen-1p7b}
  \end{subfigure}\hfill
  \begin{subfigure}[t]{0.49\linewidth}
    \centering
    \includegraphics[width=\linewidth]{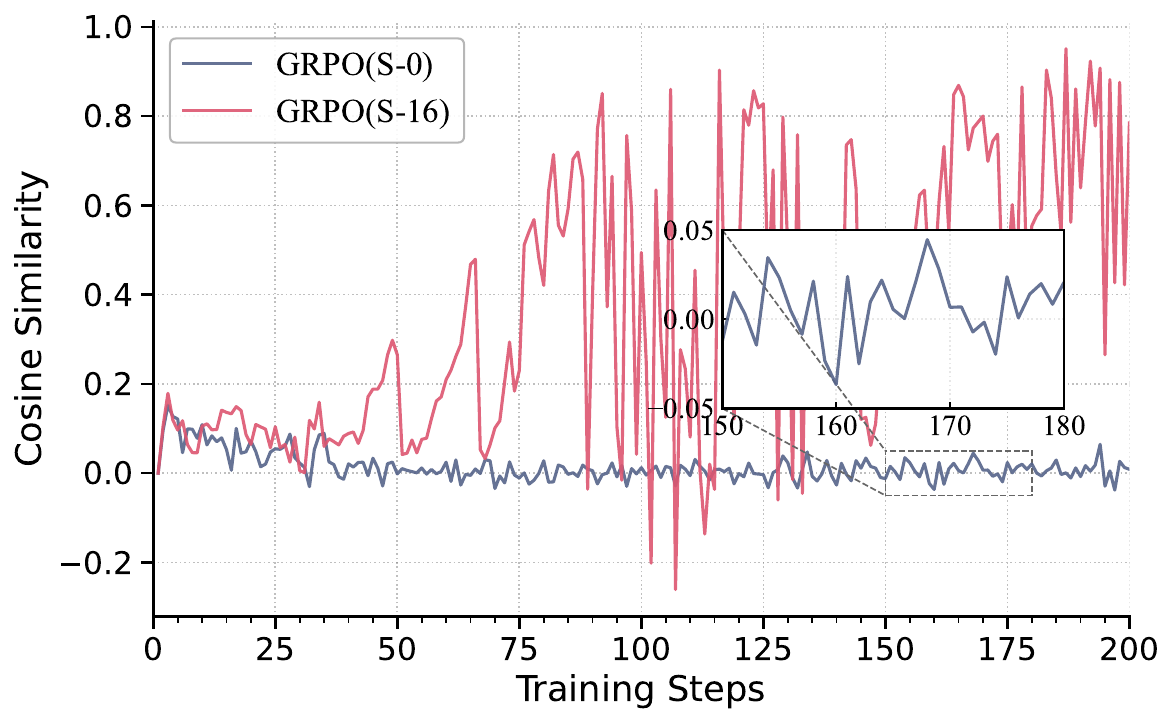}
    \caption{\small Qwen3-4B}
    \label{fig:app-ct-qwen-4b}
  \end{subfigure}

  \vspace{0.25em}

  \begin{subfigure}[t]{0.49\linewidth}
    \centering
    \includegraphics[width=\linewidth]{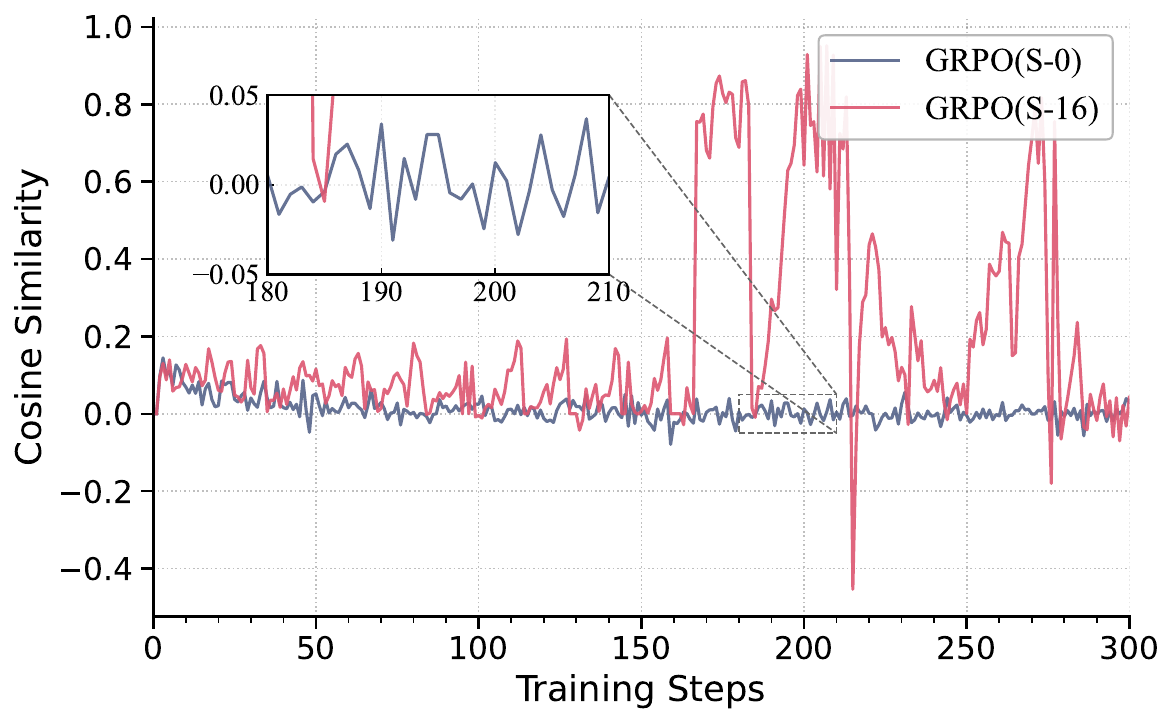}
    \caption{\small Qwen3-8B}
    \label{fig:app-ct-qwen-8b}
  \end{subfigure}\hfill
  \begin{subfigure}[t]{0.49\linewidth}
    \centering
    \includegraphics[width=\linewidth]{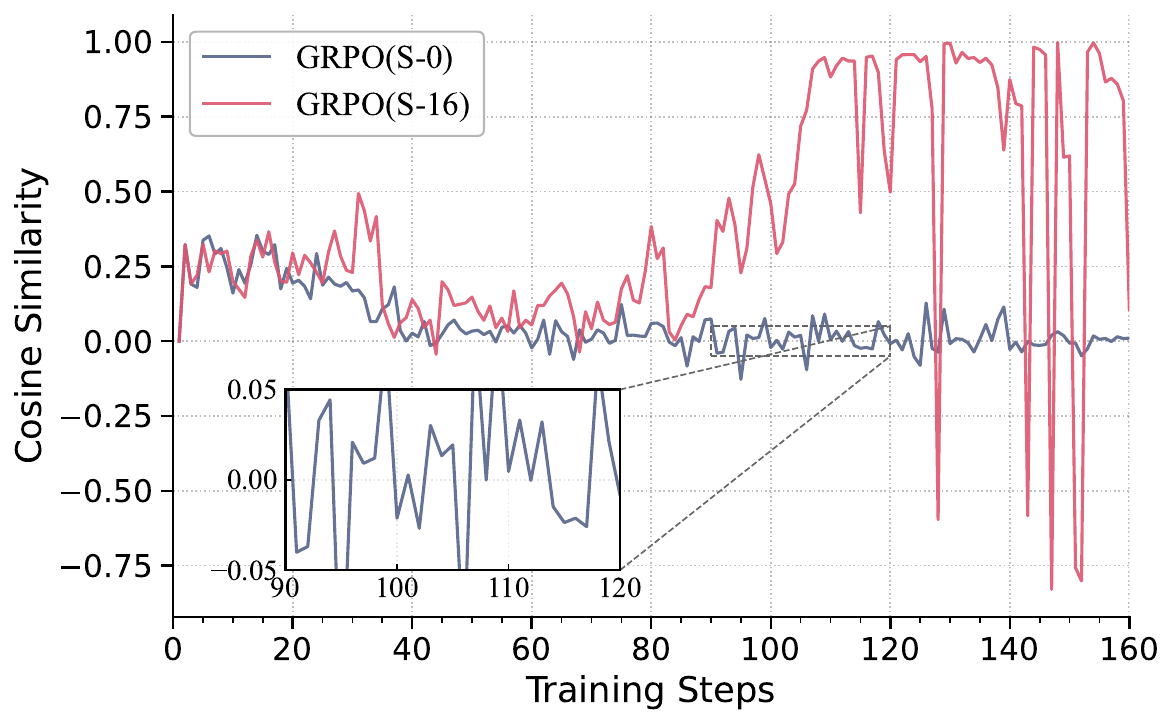}
    \caption{\small Llama-3.2-3B-Instruct}
    \label{fig:app-ct-llama-3b}
  \end{subfigure}

  \caption{
    \textbf{Consecutive-gradient cosine similarity under synchronized vs.\ stale-rollout training across scales.}
    Each panel corresponds to a different model and overlays $c_t$ for synchronized training ($s{=}0$) and stale-rollout training ($s{=}16$).
  }
  \label{fig:app-ct-s0-vs-s16-2x2}
\end{figure}

\section{Formal Convergence Analysis for Asynchronous GRPO}
\label{app:async-convergence}

\subsection{Setup}

Let $L$ denote a differentiable surrogate objective to be \emph{maximized}.
Consider the asynchronous (stale-rollout) stochastic gradient ascent iteration
\begin{equation}
\theta_{t+1} = \theta_t + \eta\,\widehat g_t,
\label{eq:update}
\end{equation}
where $\widehat g_t$ is a stochastic gradient estimator constructed from rollouts generated by a possibly stale
behavior policy (e.g., $\pi_{\theta_{t-\tau_t}}$).

Throughout, we use the standard decomposition
\begin{equation}
\widehat g_t = \nabla L(\theta_t) + b_t + \xi_t,
\label{eq:grad-decomp}
\end{equation}
where $b_t$ denotes the staleness-induced bias and $\xi_t$ is a zero-mean noise term.

\subsection{Assumptions}

\paragraph{Assumption A.1 (Smoothness).}
$L$ is $L$-smooth: for all $\theta,\theta'\in\mathbb{R}^d$,
\[
\|\nabla L(\theta)-\nabla L(\theta')\| \le L\|\theta-\theta'\|.
\]

\paragraph{Assumption A.2 (Noise).}
There exists $\sigma^2>0$ such that
\[
\mathbb{E}[\xi_t\mid\mathcal{F}_t]=0,
\qquad
\mathbb{E}\!\left[\|\xi_t\|^2\mid\mathcal{F}_t\right]\le\sigma^2 .
\]

\paragraph{Assumption A.3 (Bounded objective).}
$L^\star := \sup_{\theta} L(\theta) < \infty$.

\paragraph{Assumption A.4 (Stepsize).}
The stepsize is constant and satisfies $\eta \le \frac{1}{4L}$.

\subsection{Alignment-Aware Convergence}

\begin{theorem}[Alignment-Aware Convergence under Staleness]
\label{thm:align-aware}
Under Assumptions A.1--A.4, the iterates generated by \eqref{eq:update}--\eqref{eq:grad-decomp} satisfy, for any $T\ge1$,
\begin{align}
\min_{0\le t\le T-1}\mathbb{E}\|\nabla L(\theta_t)\|^2
\;\le\;
&\frac{2(L^\star-L(\theta_0))}{\eta T}
+2L\eta\sigma^2
+\frac{4L\eta}{T}\sum_{t=0}^{T-1}\mathbb{E}\|b_t\|^2
\nonumber\\
&\quad
-\frac{2}{T}\sum_{t=0}^{T-1}\mathbb{E}\langle \nabla L(\theta_t), b_t\rangle .
\label{eq:align-bound}
\end{align}
\end{theorem}

\paragraph{Proof.}
By $L$-smoothness,
\[
L(\theta_{t+1})
\ge
L(\theta_t)
+\eta\langle \nabla L(\theta_t),\widehat g_t\rangle
-\frac{L\eta^2}{2}\|\widehat g_t\|^2 .
\]
Taking conditional expectation and using \eqref{eq:grad-decomp},
\[
\mathbb{E}\!\left[\langle \nabla L(\theta_t),\widehat g_t\rangle\mid\mathcal{F}_t\right]
=
\|\nabla L(\theta_t)\|^2+\langle \nabla L(\theta_t),b_t\rangle .
\]
Moreover,
\[
\mathbb{E}\!\left[\|\widehat g_t\|^2\mid\mathcal{F}_t\right]
\le
2\|\nabla L(\theta_t)\|^2+2\|b_t\|^2+\sigma^2 .
\]
Combining the above and using $\eta\le 1/(4L)$ yields
\[
\frac{\eta}{2}\mathbb{E}\|\nabla L(\theta_t)\|^2
\le
\mathbb{E}[L(\theta_{t+1})-L(\theta_t)]
-\eta\mathbb{E}\langle \nabla L(\theta_t),b_t\rangle
+L\eta^2\mathbb{E}\|b_t\|^2
+\frac{L\eta^2}{2}\sigma^2 .
\]
Summing over $t=0,\dots,T-1$, telescoping, and using $L(\theta_T)\le L^\star$ gives
\[
\frac{1}{T}\sum_{t=0}^{T-1}\mathbb{E}\|\nabla L(\theta_t)\|^2
\le
\frac{2(L^\star-L(\theta_0))}{\eta T}
+2L\eta\sigma^2
+\frac{4L\eta}{T}\sum_{t=0}^{T-1}\mathbb{E}\|b_t\|^2
-\frac{2}{T}\sum_{t=0}^{T-1}\mathbb{E}\langle \nabla L(\theta_t),b_t\rangle .
\]
The result follows by $\min_t a_t\le \frac{1}{T}\sum_t a_t$.
\hfill$\square$

\subsection{Temporal Persistence and Gradient Alignment}

We now express the alignment term in \eqref{eq:align-bound} through temporal coupling induced by staleness.

\paragraph{Assumption A.5 (Bounded staleness and drift).}
At time $t$, the gradient is computed using rollouts generated at $\theta_{t-\tau_t}$ with $\tau_t\le S$, so that
\[
\theta_t-\theta_{t-\tau_t}=\sum_{k=1}^{\tau_t}\eta\,\widehat g_{t-k}.
\]

\paragraph{Assumption A.6 (Small-drift linearization).}
There exist matrices $B_t$ and remainders $r_t$ such that
\[
b_t = B_t(\theta_t-\theta_{t-\tau_t}) + r_t,
\qquad
\mathbb{E}\|r_t\|\le \varepsilon_t .
\]

\paragraph{Assumption A.7 (Temporal persistence).}
There exists $\lambda_t\ge0$ such that
\[
\mathbb{E}\!\left[\widehat g_{t-1}^\top B_t \widehat g_{t-1}\right]
\ge
\lambda_t\,\mathbb{E}\|\widehat g_{t-1}\|^2 .
\]

\paragraph{Assumption A.8 (One-step alignment stability).}
There exist $\rho_t\in[0,1]$ and $\delta_t\ge0$ such that
\[
\mathbb{E}\!\left[\langle \nabla L(\theta_t), B_t\widehat g_{t-1}\rangle\right]
\ge
\rho_t\,\mathbb{E}\!\left[\widehat g_{t-1}^\top B_t \widehat g_{t-1}\right]
-\delta_t .
\]

\paragraph{Discussion of assumptions and empirical validation.}
Assumptions A.5–A.8 formalize the minimal structure needed to capture how staleness alters optimization dynamics in asynchronous GRPO. Assumption A.5 is purely structural, expressing parameter drift as the accumulation of intervening updates under bounded staleness, and holds by construction in practical asynchronous pipelines. Assumption A.6 introduces a local linearization of the staleness-induced bias, which is standard in delayed-gradient analyses and is justified as long as parameter drift remains moderate—precisely the regime prior to collapse. Assumptions A.7 and A.8 characterize temporal persistence: they require that the bias induced by stale rollouts retains nontrivial alignment with recent update directions and that this alignment does not vanish abruptly from one step to the next. Importantly, these assumptions are directional and expectation-based, rather than global or worst-case.

Our empirical results provide direct evidence supporting these assumptions. Across model scales and benchmarks, asynchronous GRPO consistently exhibits persistently elevated cosine similarity between consecutive gradients, indicating that stale gradients systematically reinforce previous update directions (supporting Assumption A.7). Moreover, the smooth evolution of alignment prior to collapse suggests that this persistence is stable across successive steps, consistent with Assumption A.8. Finally, the sharp rise in gradient alignment immediately preceding training collapse coincides with large parameter drift and breakdown of local linear behavior, aligning with the interpretation that Assumption A.6 holds in the stable regime and fails precisely when instability emerges. Together, these observations validate the assumptions as accurate descriptors of the dynamics leading to collapse, rather than as restrictive conditions imposed for analytical convenience.

\begin{lemma}[Alignment Expansion]
\label{lem:alignment-expansion}
Under Assumptions A.5--A.6,
\[
\langle \nabla L(\theta_t), b_t\rangle
=
\eta\langle \nabla L(\theta_t), B_t\widehat g_{t-1}\rangle
+
\eta\sum_{k=2}^{\tau_t}\langle \nabla L(\theta_t), B_t\widehat g_{t-k}\rangle
+
\langle \nabla L(\theta_t), r_t\rangle .
\]
\end{lemma}

\paragraph{Proof.}
Substitute the drift expansion into $b_t$ and expand the inner product.
\hfill$\square$

\begin{theorem}[Convergence of Asynchronous GRPO with Alignment Persistence]
\label{thm:align-async}
Under Assumptions A.1--A.8, the iterates satisfy
\begin{align}
\min_{0\le t\le T-1}\mathbb{E}\|\nabla L(\theta_t)\|^2
\;\le\;
&\frac{2(L^\star-L(\theta_0))}{\eta T}
+2L\eta\sigma^2
+\frac{4L\eta}{T}\sum_{t=0}^{T-1}\mathbb{E}\|b_t\|^2
\nonumber\\
&\quad
-\frac{2}{T}\sum_{t=0}^{T-1}\eta\,\rho_t\lambda_t
\mathbb{E}\|\widehat g_{t-1}\|^2
+\mathrm{Err}(T),
\label{eq:final-async-bound}
\end{align}
where
\[
\mathrm{Err}(T)
=
\frac{2}{T}\sum_{t=0}^{T-1}
\Bigg(
\eta\sum_{k=2}^{\tau_t}
\mathbb{E}\big|\langle \nabla L(\theta_t), B_t\widehat g_{t-k}\rangle\big|
+
\mathbb{E}\|\nabla L(\theta_t)\|\|r_t\|
+
\eta\delta_t
\Bigg).
\]
\end{theorem}

\paragraph{Proof.}
Apply Lemma~\ref{lem:alignment-expansion} inside Theorem~\ref{thm:align-aware}.
The leading term is lower bounded using Assumptions A.7--A.8, while the remaining terms are controlled
by absolute values and collected into $\mathrm{Err}(T)$.
\hfill$\square$

\section{Advantage of \textsc{GAC}: Bias Suppression via Alignment Control}
\label{app:prop}

The convergence bounds in Theorems~\ref{thm:align-aware}--\ref{thm:align-async} show that staleness impacts optimization primarily through the bias term $b_t$.
In particular, the penalty $\sum_t \mathbb{E}\|b_t\|^2$ is always detrimental and can dominate the bound under large staleness.
Under temporal persistence (Assumptions~A.5--A.8), the leading component of $b_t$ is aligned with the previous update direction, which motivates \textsc{GAC}'s strategy of attenuating progress along this stale-aligned direction.
We formalize this intuition by showing that projecting $b_t$ away from $\mathrm{span}(\widehat g_{t-1})$ yields a strict reduction in the squared bias magnitude, up to higher-order approximation terms.

\begin{proposition}[Bias reduction via stale-direction projection]
\label{prop:bias_reduction_formal}
Fix an iteration $t\ge 1$ and let $\widehat g_{t-1}$ be the previous stochastic gradient.
Define the reference direction
\[
u_{t-1} \;\triangleq\;
\begin{cases}
\widehat g_{t-1}/\|\widehat g_{t-1}\|_2, & \widehat g_{t-1}\neq 0,\\
0, & \widehat g_{t-1}=0.
\end{cases}
\]
Let $P^\perp_{t-1}\triangleq I-u_{t-1}u_{t-1}^\top$ denote the orthogonal projector onto $\mathrm{span}(u_{t-1})^\perp$.

Assume the staleness-induced bias $b_t$ admits the \emph{one-step persistent linearization}
\begin{equation}
b_t \;=\; \eta_{t-1} B_t \widehat g_{t-1} + r_t,
\label{eq:one_step_bias_lin}
\end{equation}
where $B_t\in\mathbb{R}^{d\times d}$ is (possibly random) and $r_t$ is a remainder term.
Assume further the following \emph{temporal persistence} condition holds:
\begin{equation}
\widehat g_{t-1}^\top B_t \widehat g_{t-1}
\;\ge\;
\lambda_t \|\widehat g_{t-1}\|_2^2
\qquad\text{almost surely},
\label{eq:temporal_persistence_strong}
\end{equation}
for some $\lambda_t>0$.

Define the \emph{projected (effective) bias} as
\[
b_t^\perp \;\triangleq\; P^\perp_{t-1} b_t.
\]
Then the following inequality holds:
\begin{align}
\mathbb{E}\|b_t^\perp\|_2^2
\;\le\;
\mathbb{E}\|b_t\|_2^2
\;-\;
\eta_{t-1}^2\lambda_t^2\,\mathbb{E}\|\widehat g_{t-1}\|_2^2
\;+\;
2\eta_{t-1}\lambda_t\,\mathbb{E}\!\left[\|\widehat g_{t-1}\|_2\,\|r_t\|_2\right].
\label{eq:bias_reduction_explicit}
\end{align}
In particular, if $\mathbb{E}\|r_t\|_2^2 \le \epsilon_t^2$ for some $\epsilon_t\ge 0$, then
\begin{align}
\mathbb{E}\|b_t^\perp\|_2^2
\;\le\;
\mathbb{E}\|b_t\|_2^2
\;-\;
\eta_{t-1}^2\lambda_t^2\,\mathbb{E}\|\widehat g_{t-1}\|_2^2
\;+\;
2\eta_{t-1}\lambda_t\,\epsilon_t\,\sqrt{\mathbb{E}\|\widehat g_{t-1}\|_2^2}.
\label{eq:bias_reduction_eps}
\end{align}
\end{proposition}

\begin{proof}
If $\widehat g_{t-1}=0$, then $u_{t-1}=0$ and $P^\perp_{t-1}=I$, so $b_t^\perp=b_t$ and
\eqref{eq:bias_reduction_explicit} holds trivially. We assume $\widehat g_{t-1}\neq 0$ below.

\paragraph{Step 1: Pythagorean identity for orthogonal projection.}
Since $P^\perp_{t-1}$ is an orthogonal projector, we have the decomposition
\[
b_t \;=\; (u_{t-1}u_{t-1}^\top)b_t \;+\; (I-u_{t-1}u_{t-1}^\top)b_t
\;=\; (u_{t-1}u_{t-1}^\top)b_t \;+\; b_t^\perp,
\]
and the two components are orthogonal. Therefore,
\begin{equation}
\|b_t\|_2^2 \;=\; \|b_t^\perp\|_2^2 \;+\; \|(u_{t-1}u_{t-1}^\top)b_t\|_2^2.
\label{eq:pythagoras}
\end{equation}
Moreover,
\[
(u_{t-1}u_{t-1}^\top)b_t
= \langle b_t,u_{t-1}\rangle u_{t-1},
\qquad
\|(u_{t-1}u_{t-1}^\top)b_t\|_2^2
= \langle b_t,u_{t-1}\rangle^2.
\]
Plugging this into \eqref{eq:pythagoras} yields the exact identity
\begin{equation}
\|b_t^\perp\|_2^2
\;=\;
\|b_t\|_2^2
\;-\;
\langle b_t,u_{t-1}\rangle^2
\;=\;
\|b_t\|_2^2
\;-\;
\frac{\langle b_t,\widehat g_{t-1}\rangle^2}{\|\widehat g_{t-1}\|_2^2}.
\label{eq:bias_proj_identity}
\end{equation}

\paragraph{Step 2: Lower bound the removed component using persistence.}
Using the one-step linearization \eqref{eq:one_step_bias_lin},
\[
\langle b_t,\widehat g_{t-1}\rangle
=
\eta_{t-1}\,\widehat g_{t-1}^\top B_t \widehat g_{t-1}
\;+\;
\langle r_t,\widehat g_{t-1}\rangle.
\]
By the persistence condition \eqref{eq:temporal_persistence_strong},
\[
\eta_{t-1}\,\widehat g_{t-1}^\top B_t \widehat g_{t-1}
\;\ge\;
\eta_{t-1}\lambda_t\|\widehat g_{t-1}\|_2^2.
\]
Also, by Cauchy--Schwarz,
\[
\langle r_t,\widehat g_{t-1}\rangle \ge -\|r_t\|_2\,\|\widehat g_{t-1}\|_2.
\]
Combining the two displays gives the deterministic bound
\begin{equation}
\langle b_t,\widehat g_{t-1}\rangle
\;\ge\;
\eta_{t-1}\lambda_t\|\widehat g_{t-1}\|_2^2
\;-\;
\|r_t\|_2\,\|\widehat g_{t-1}\|_2.
\label{eq:inner_lower}
\end{equation}

\paragraph{Step 3: Square and normalize.}
Divide \eqref{eq:inner_lower} by $\|\widehat g_{t-1}\|_2$ and square both sides:
\[
\frac{\langle b_t,\widehat g_{t-1}\rangle^2}{\|\widehat g_{t-1}\|_2^2}
\;\ge\;
\Bigl(\eta_{t-1}\lambda_t\|\widehat g_{t-1}\|_2 - \|r_t\|_2\Bigr)^2
=
\eta_{t-1}^2\lambda_t^2\|\widehat g_{t-1}\|_2^2
-2\eta_{t-1}\lambda_t\|\widehat g_{t-1}\|_2\|r_t\|_2
+\|r_t\|_2^2.
\]
Dropping the nonnegative last term $\|r_t\|_2^2$ yields
\begin{equation}
\frac{\langle b_t,\widehat g_{t-1}\rangle^2}{\|\widehat g_{t-1}\|_2^2}
\;\ge\;
\eta_{t-1}^2\lambda_t^2\|\widehat g_{t-1}\|_2^2
-2\eta_{t-1}\lambda_t\|\widehat g_{t-1}\|_2\|r_t\|_2.
\label{eq:removed_lb}
\end{equation}

\paragraph{Step 4: Plug back into the projection identity and take expectation.}
Substitute \eqref{eq:removed_lb} into \eqref{eq:bias_proj_identity}:
\[
\|b_t^\perp\|_2^2
\le
\|b_t\|_2^2
-
\eta_{t-1}^2\lambda_t^2\|\widehat g_{t-1}\|_2^2
+
2\eta_{t-1}\lambda_t\|\widehat g_{t-1}\|_2\|r_t\|_2.
\]
Taking expectation gives \eqref{eq:bias_reduction_explicit}.

Finally, if $\mathbb{E}\|r_t\|_2^2 \le \epsilon_t^2$, then by Cauchy--Schwarz,
\[
\mathbb{E}\!\left[\|\widehat g_{t-1}\|_2\|r_t\|_2\right]
\le
\sqrt{\mathbb{E}\|\widehat g_{t-1}\|_2^2}\;\sqrt{\mathbb{E}\|r_t\|_2^2}
\le
\epsilon_t\sqrt{\mathbb{E}\|\widehat g_{t-1}\|_2^2},
\]
which yields \eqref{eq:bias_reduction_eps}.
\end{proof}

\section{Per-Benchmark Performance Analysis}
\label{app:benchmark-analysis}

Figure~\ref{fig:app-benchmark-breakdown-qwen4b} reports a per-benchmark breakdown of final validation accuracy for Qwen3-4B under stale-rollout training ($s{=}16$).
Across all seven benchmarks, \textsc{GAC} consistently outperforms all off-policy baselines across benchmarks.
Moreover, \textsc{GAC} achieves performance that is highly comparable to synchronized GRPO: the per-benchmark gaps remain small and do not exhibit systematic deterioration under asynchrony.
Importantly, this recovery does not come from trading off learning speed for stability---\textsc{GAC} preserves smooth, stable convergence while maintaining a convergence rate that closely tracks the synchronized reference throughout training.

\begin{figure*}[p]
  \centering
  % ---------------- Row 1 (2 panels): AMC ----------------
  \begin{subfigure}[t]{0.48\textwidth}
    \centering
    \includegraphics[width=\linewidth]{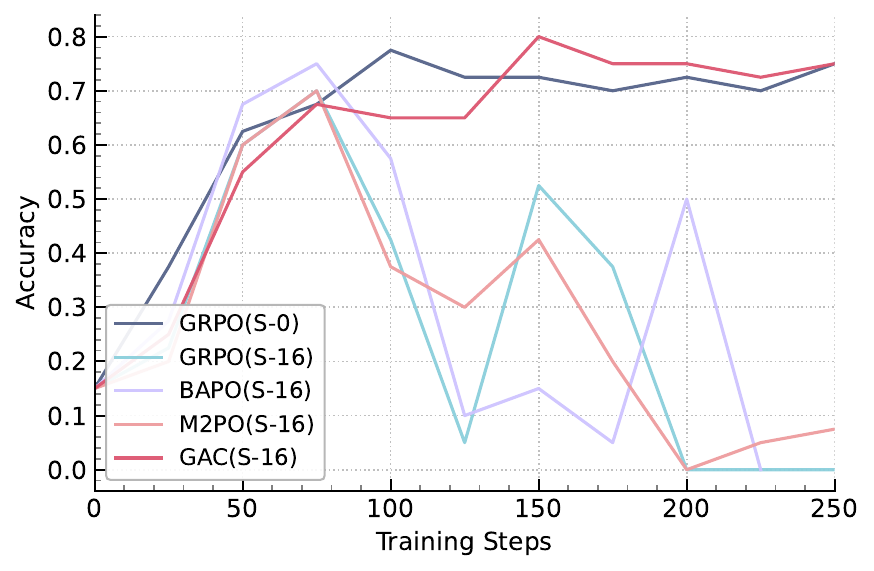}
    \caption{AMC23}
    \label{fig:app-perbench-amc23}
  \end{subfigure}\hfill
  \begin{subfigure}[t]{0.48\textwidth}
    \centering
    \includegraphics[width=\linewidth]{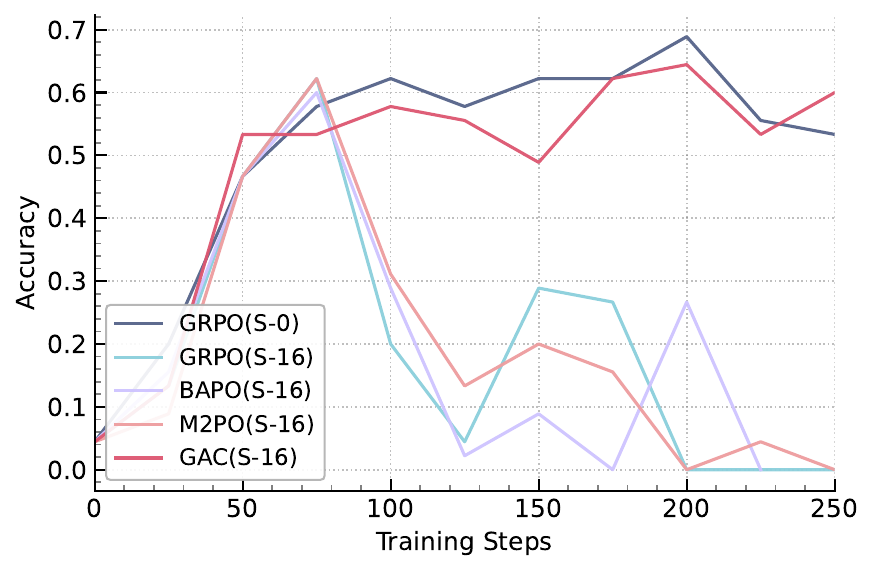}
    \caption{AMC24}
    \label{fig:app-perbench-amc24}
  \end{subfigure}

  \vspace{0.7em}

  % ---------------- Row 2 (2 panels): Olympiad + Minerva ----------------
  \begin{subfigure}[t]{0.48\textwidth}
    \centering
    \includegraphics[width=\linewidth]{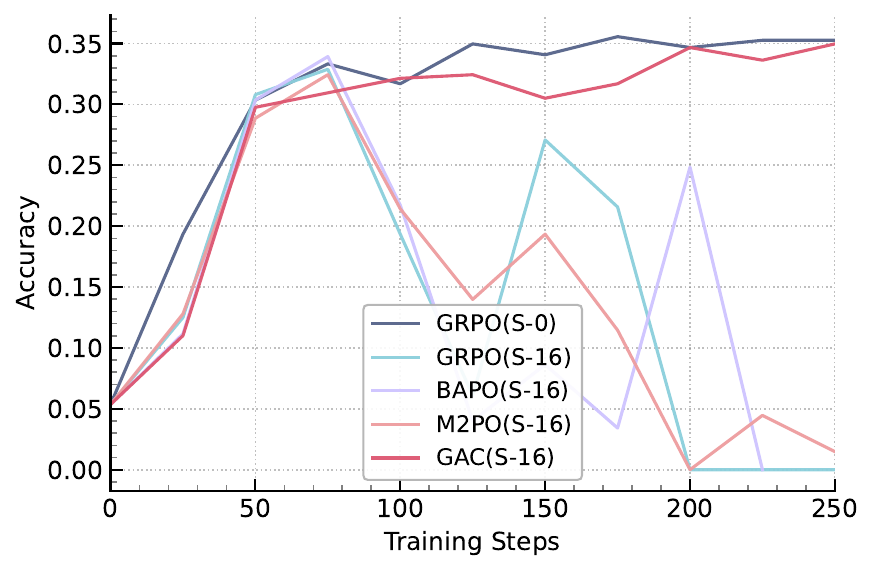}
    \caption{OlympiadBench}
    \label{fig:app-perbench-olymp}
  \end{subfigure}\hfill
  \begin{subfigure}[t]{0.48\textwidth}
    \centering
    \includegraphics[width=\linewidth]{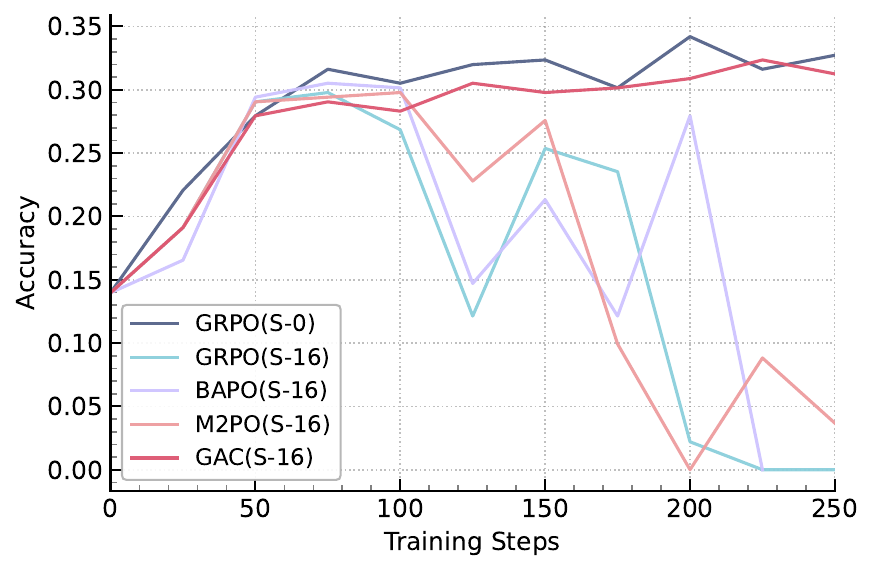}
    \caption{MinervaMath}
    \label{fig:app-perbench-minerva}
  \end{subfigure}

  \vspace{0.7em}

  % ---------------- Row 3 (3 panels): AIME + Math ----------------
  \begin{subfigure}[t]{0.32\textwidth}
    \centering
    \includegraphics[width=\linewidth]{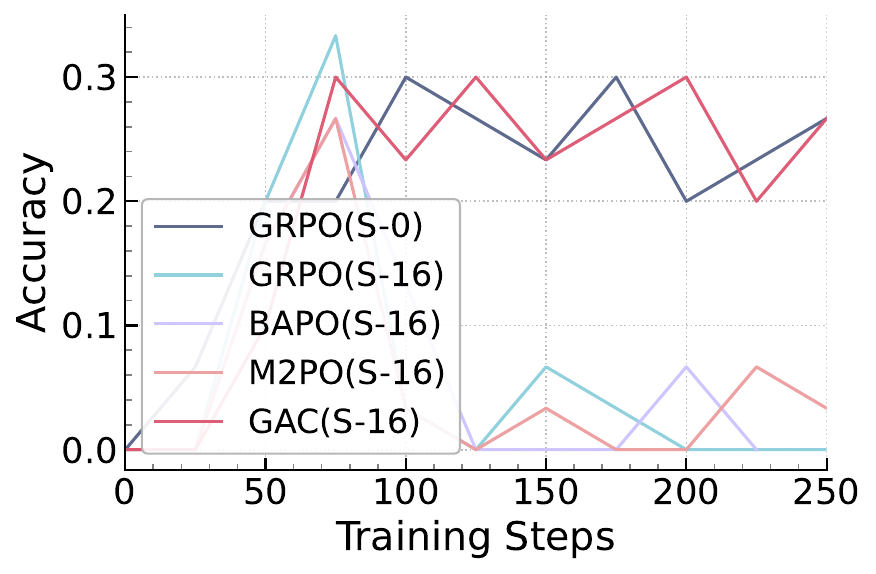}
    \caption{AIME24}
    \label{fig:app-perbench-aime24}
  \end{subfigure}\hfill
  \begin{subfigure}[t]{0.32\textwidth}
    \centering
    \includegraphics[width=\linewidth]{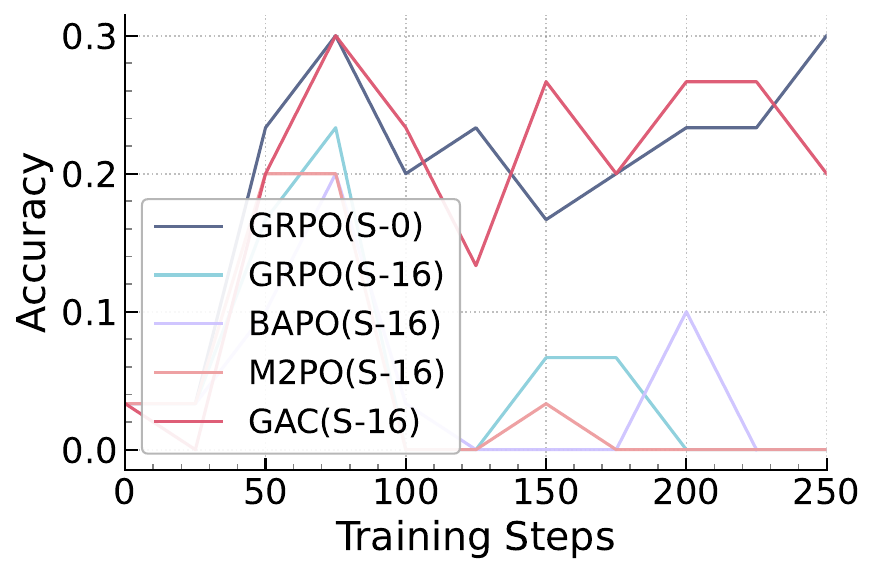}
    \caption{AIME25}
    \label{fig:app-perbench-aime25}
  \end{subfigure}\hfill
  \begin{subfigure}[t]{0.32\textwidth}
    \centering
    \includegraphics[width=\linewidth]{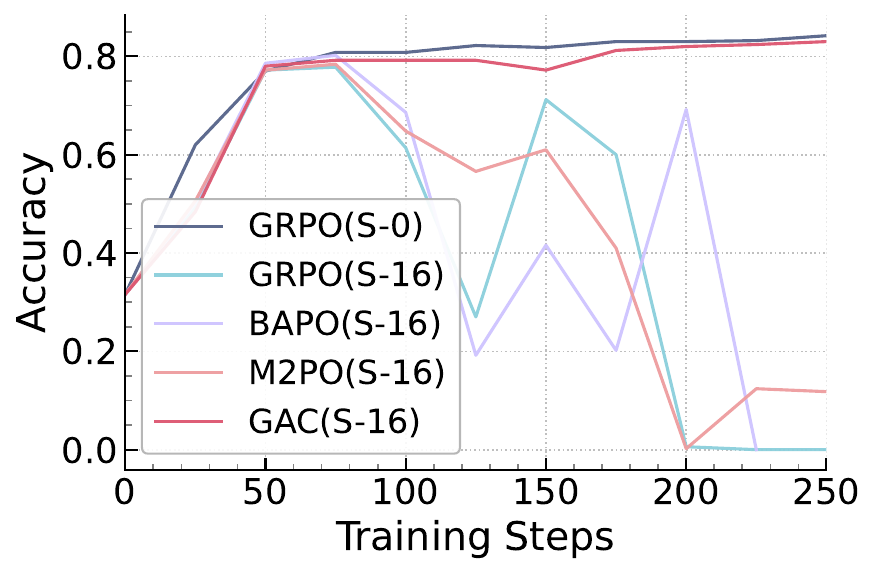}
    \caption{Math500}
    \label{fig:app-perbench-math500}
  \end{subfigure}

  \caption{\textbf{Per-benchmark accuracy breakdown for Qwen3-4B under staleness ($s{=}16$).}
  Each panel reports final validation accuracy on one benchmark, comparing \textsc{GAC} against synchronized GRPO (on-policy reference) and off-policy baselines under stale rollouts.}
  \label{fig:app-benchmark-breakdown-qwen4b}
\end{figure*}

%%%%%%%%%%%%%%%%%%%%%%%%%%%%%%%%%%%%%%%%%%%%%%%%%%%%%%%%%%%%%%%%%%%%%%%%%%%%%%%
%%%%%%%%%%%%%%%%%%%%%%%%%%%%%%%%%%%%%%%%%%%%%%%%%%%%%%%%%%%%%%%%%%%%%%%%%%%%%%%

% \clearpage
% \input{format/checklist}

\end{document}